\documentclass[nohyperref]{article}%
\usepackage{amsmath}%
\usepackage{amsfonts}%
\usepackage{amssymb}%
\usepackage{graphicx}
\usepackage{mathtools} 
\usepackage{extarrows} 
\usepackage{bbm}
\usepackage{bm}
\usepackage{systeme}
\usepackage{color}
\usepackage{amsthm}
\usepackage[hidelinks,colorlinks=true,linkcolor=blue,citecolor=blue]{hyperref}
\usepackage{natbib}

\usepackage{microtype}
\usepackage{subfigure}
\usepackage{booktabs} 

\theoremstyle{plain}
\newtheorem{theorem}{Theorem}[section]
\newtheorem{proposition}[theorem]{Proposition}
\newtheorem{lemma}[theorem]{Lemma}

\theoremstyle{definition}
\newtheorem{definition}[theorem]{Definition}
\newtheorem{assumption}[theorem]{Assumption}
\theoremstyle{remark}
\newtheorem{remark}[theorem]{Remark}

\let\hat\widehat
\let\tilde\widetilde

\newcommand{\N}{\ensuremath{\mathbb{N}}} 
\newcommand{\R}{\ensuremath{\mathbb{R}}} 

\newcommand{\poly}{{\rm poly}}
\newcommand{\notshow}[1]{{}}

\DeclareMathOperator*{\E}{\mathbb{E}}

\def \PP  {{\cal P}}

\def \NN  {{\cal N}}
\def \LL  {{\cal L}}
\def \BB  {{\cal B}}
\def \VV  {{\cal V}}
\def \EE  {{\cal E}}

\def \cD{{\mathcal{D}}}
\def \cC{{\mathcal{C}}}
\def \cF{{\mathcal{F}}}

\def \S{{\mathcal{S}}}
\def \AA{{\mathcal{A}}}
\def \NN{{\mathcal{N}}}
\def \CC{{\mathcal{C}}}
\def \ZZ{{\mathcal{Z}}}
\def \FF{{\mathcal{F}}}
\def \LL{{\mathcal{L}}}
\def \XX{{\mathcal{X}}}
\def \reg{{\textnormal{Regret}}}
\def \gap{{\textnormal{gap}}}

\newenvironment{prevproof}[2]{\noindent {\em {Proof of {#1}~\ref{#2}:}}}{$\Box$\vskip \belowdisplayskip}

\usepackage{comment}

\usepackage{hyperref}
\usepackage{algorithm}
\usepackage{algorithmic}

\begin{document}

\title{The Best of Both Worlds: Reinforcement Learning with Logarithmic Regret  and  Policy Switches}

\author{ Grigoris Velegkas \\Yale University, USA\\grigoris.velegkas@yale.edu
 \and Zhuoran Yang\\Yale University, USA\\zhuoran.yang@yale.edu
  \and Amin Karbasi\\Yale University, USA\\ amin.karbasi@yale.edu
}

\date{}
\maketitle

\begin{abstract}
    In this paper, we study the problem of regret minimization for episodic
Reinforcement Learning (RL) both in the model-free
and the model-based setting.
We focus on learning with general function classes
and general model classes, and we derive results that scale
with the \emph{eluder dimension} of these classes. In contrast to the existing body of work that mainly establishes instance-independent regret guarantees, we focus on the instance-dependent setting and show that the regret scales logarithmically with the horizon $T$, provided that there is a gap between the best and the second best
action in every state. In addition, we show that such a logarithmic regret bound is realizable by algorithms with $O(\log T)$ switching cost (also known as \emph{adaptivity complexity}). In other words, these algorithms rarely switch their policy during the course of their execution.  Finally, we complement our results
with lower bounds which show that even in the tabular setting,
we cannot hope for regret guarantees lower than $o(\log T)$.

\end{abstract}

    \section{Introduction} \label{sec:intro} 

 The main goal of Reinforcement Learning (RL) is the design and analysis
 of algorithms for automated
decision making in complex and unknown environments. The environment
is modeled as a \emph{state space} and the available decisions
are modeled as an \emph{action space}. In recent years, RL has 
seen tremendous success in practical applications including, but not
limited to, games and robotics \citep{mnih2015human, silver2016mastering, duan2016benchmarking, silver2017mastering, vinyals2019grandmaster}.
Despite this success, a theoretical understanding of the 
algorithms that are
deployed in these settings remains elusive. Traditionally,
theoretical RL approaches have focused on the \emph{tabular}
setting where the complexity of the algorithms depends
on the cardinality of the
aforementioned spaces \citep{sutton2018reinforcement}.
Thus, they are not suitable for applications where
the state-action space is very large. 

A different approach that has gained a lot of attention recently,
is the \emph{function approximation} regime where the cumulative
reward of the algorithm is modeled through a function, such
as linear functions over some feature space.
The advantage of this approach is that the algorithm
can perform its search over a lower-dimensional space.
There is a long line of work that provides regret
guarantees for RL in the function approximation setting \citep{osband2014model, osband2016generalization,yang2020reinforcement,jin2020provably,ayoub2020model, cai2020provably,kakade2020information, zanette2020frequentist, he2021logarithmic,kong2021online, zhou2021provably}.

Most of these works have focused on establishing worst-case
$\sqrt{T}$-regret guarantees, where $T$ is the number
of interactions with the environment. The caveat with these guarantees
is that they are pessimistic since they neglect benign 
settings where even an \emph{exponential} improvement
over these bounds is achievable. To address this issue,
there are some works that obtain instance-dependent
regret bounds for RL in the tabular setting and 
in the linear function approximation 
setting~\citep{simchowitz2019non, yang2020reinforcement, he2021logarithmic}.
However, getting logarithmic regret bounds in the context of 
\emph{general} function approximation remains open. Hence,
a natural question is the following:

\emph{Can we establish instance-dependent logarithmic regret
bounds with general function approximation?}

In this paper, we provide an affirmative answer to this
question. Following the assumptions in the existing
literature, we model the RL problem as an MDP that enjoys
the property that the optimal \emph{policy} is at least
$\gap_{\min}$ better than any other policy,
where $ \gap_{\min} > 0$ is a parameter that
 captures the hardness of the underlying problem.
We focus both on the \emph{model-free} and \emph{model-based}
settings with general function approximation, where
we represent the value function or the transition model
by a given function class, respectively.

In the model-free setting, we study the algorithm
proposed in~\citet{kong2021online}, which is a variant
of the least-squares value iteration (LSVI) with upper confidence
bound (UCB) bonuses
that guide exploration. Here, the bonus functions
are given by the width of a data-dependent confidence region
for LSVI. For the model-based setting, we develop 
a similar algorithm which combines value-targeted
regression with UCB bonuses. On top of the 
logarithmic regret
guarantees they enjoy, our algorithms feature 
lazy policy updates, in the sense that policy
is updated rarely and only when certain conditions are met.

For both settings, we establish  $O(\poly(\log T) \cdot\poly(H) \cdot  \poly(d_{\cF} )  \cdot 1/\gap_{\min})$
regret guarantees, where $T$ is
the number of interactions with the environment,
$H$ is the \emph{planning horizon}, $d_{\cF} $
is a term that captures the complexity of the function class $\cF$ which is used to approximate either the value function or the transition model. In particular, $d_{\cF}$ involves both the 
\emph{eluder dimension}~\citep{russo2013eluder} and the log-covering numbers of   the 
 function classes.
That is, for benign
MDPs where $\gap_{\min} > 0$ these RL algorithms
achieve logarithmic regret, which is exponentially
better than the worst-case $O(\sqrt{T})$-regret.
Moreover, we show that the \emph{adaptivity complexity},
meaning the number of different policies our algorithms
use, is also logarithmic in $T$.
To the best of our knowledge, this is the first
work that establishes a logarithmic
instance-dependent regret guarantee for RL
with general function approximation.

\subsection{Related Work}
\label{sec:related work}

{\noindent \bf Logarithmic regret bounds for bandits.} 
There is a 
long line of work that establishes logarithmic
regret guarantees in bandit problems. Essentially,
bandits are a special case of RL where the transition to the
next state does not depend on the action that was taken by
the agent. An extensive list of such algorithms can
be found in~\cite{bubeck2012regret, slivkins2019introduction, lattimore2020bandit}.

{\noindent \bf Logarithmic regret bounds for RL.} 
A series of works are devoted to proving instance-dependent
logarithmic regret bounds in tabular RL. \citet{ok2018exploration, simchowitz2019non} prove lower bounds that show
that a logarithmic dependence on $T$ is unavoidable.
Considering the upper bounds,~\citet{auer2007logarithmic, tewari2007optimistic} establish logarithmic regret guarantees in the 
average reward setting. Both of these guarantees are asymptotic
as they require the number of interactions $T$ with the MDP to 
be large enough. Regarding non-asymptotic bounds,~\citet{jaksch2010near}
provide such an algorithm that achieves 
$O(D^2 |\S|^2 |\AA| \log(T) /\gap_{\min})$ regret for the
average-reward MDP, where $D$ is the diameter of the MDP.
For episodic MDPs, logarithmic regret upper bounds
are established in~\cite{simchowitz2019non, yang2021q}.
The work that is probably the most closely related to ours
is~\cite{he2021logarithmic}. It provides instance-dependent
logarithmic regret guarantees both in the model-free
and model-based setting with \emph{linear} function
approximation. Moreover, the proposed algorithms update
the policy in every episode.
Our work generalizes these results since the linear
regime is a special case 
of the setting we are studying. Furthermore, we 
achieve an exponential improvement on the adaptivity complexity
over their algorithms.

{\noindent \bf Bandits with limited adaptivity complexity.}
There is a lot of interest in
obtaining bandit algorithms that update their policies
rarely~\citep{abbasi2011improved, perchet2016batched,agarwal2017learning,gao2019batched, dong2020multinomial, chen2020minimax, ruan2021linear}.
Notably,~\citet{dong2020multinomial} study rare policy
switching constraints for a broader class of online 
learning and decision making problems such as logit 
bandits.

{\noindent \bf RL with limited adaptivity complexity.}
Recently, there has been a lot of interest in developing
RL algorithms that achieve sub-linear regret and have
low adaptivity complexity~\citep{bai2019provably,zhang2020almost, wang2021provably, kong2021online, gao2021provably}. 
We develop an algorithm with low adaptivity complexity
that works
in the model-based, general function approximation setting.

{\noindent \bf RL with general function approximation.}
As we have alluded to already, because of the enormous size
of the state-action space in real world applications, it 
is important to develop and analyze algorithms in the 
function approximation regime. So far, the most commonly studied
setting is RL with linear function approximation~\citep{yang2020reinforcement, jin2020provably, du2020agnostic, wang2020reinforcement, zanette2020learning, agarwal2020flambe}.
Recently, there are also important results in RL
with general function approximation. To be specific, ~\citet{jiang2017contextual} design 
an efficient algorithm whose sample complexity is bounded in terms
of the Bellman rank of the function class.~\citet{ayoub2020model}
develop an algorithm for model-based RL, whose regret bound
depends on the eluder dimension of the underlying class of models.
\citet{jin2021bellman} propose  an algorithm for problems where
the underlying class has bounded Bellman eluder dimension. 
Recent works such as \cite{wang2020reinforcement, kong2021online} develop algorithms
in the model-free setting whose regret scale with
the eluder dimension of the functions.~\citet{foster2020instance} also 
propose an LSVI-based algorithm whose regret depends on a disagreement
coefficient, which is upper bounded by the eluder dimension of the function
class. The caveat with their approach is that it requires the block MDP
assumption.


 
\section{Preliminaries}

In this section, we present the notation and the important
definitions we 
use throughout the paper.

\subsection{Notation}
\label{sec:prelim-notation}

We use the common notation $[N] = \{1,2,\ldots,N\}$.
We also define the infinity norm of some function $f:\XX \rightarrow \R$, 
where $\XX$ is some domain, to be $||f||_{\infty} = \sup_{x\in \XX} |f(x)|$.
For a dataset $\cD = \{(x_i, q_i)\}_{i=1}^n \subseteq \XX\times\R$ and a function $f:\XX \rightarrow \R$, 
we define the following norm
\begin{align*}
    ||f||_{\cD} = \left(\sum_{i=1}^n (f(x_i) - q_i)^2 \right)^{1/2}.
\end{align*}
Given a set $\ZZ = \{x_i\}_{i=1}^n \subseteq \XX$ we let
\begin{align*}
    ||f||_{\ZZ} = \left(\sum_{i=1}^n f(x_i)^2\right)^{1/2}.
\end{align*}
be the data-dependent norm.
Given a measurable set $\XX$, we denote with 
$\Delta(\XX)$ the 
probability simplex over $\XX$.
We also denote by $\mathbbm{1}[\EE]$ the indicator function
of the event $\EE$.
We denote by $\poly(x)$ a function that is a polynomial in $x$.
\subsection{Episodic Markov Decision Processes}
\label{sec:prelim-episodic-mdp}

The learning agent interacts with the environment over a sequence
of $K$ rounds which we call \emph{episodes}. 
We model the interaction of the agent with the environment in every episode as a 
\emph{Markov Decision Process} (MDP). We denote an MDP by
$M = (\S, \AA, P, r, H, s_1)$, where $\S$ is the \emph{state space}, $\AA$ is the \emph{action space}, $P = \{P_h:\S\times\AA \rightarrow \Delta(\S)\}_{h=1}^H$ are the \emph{transition kernels},
$r = \{r_h: \S\times \AA \rightarrow [0,1] \}_{h=1}^H$ 
are the \emph{reward functions} which we assume to be deterministic, $H$ is 
the \emph{planning horizon}, which is the length of every episode,
and $s_1$ is the initial state of every episode. During every episode,
the agent uses a \emph{policy} $\pi = \{\pi_h:\S \rightarrow \AA\}_{h=1}^H$, to take an action at a given state. 
We use the Q-function and V-function to evaluate the expected total
reward
generated by a policy $\pi$. More specifically,
we define
\begin{align*}
    Q_h^{\pi}(s,a) &= \E \left[\sum_{h'=h}^H r_{h'}(s_{h'}, a_{h'}) \big| s_h = s, a_h = a, \pi\right]\\
        V_h^{\pi}(s) &= \E \left[\sum_{h'=h}^H r_{h'}(s_{h'}, a_{h'}) \big| s_h = s, \pi\right],
\end{align*}
where the actions are picked according to $\pi$ and $s_{h'+1} \sim P_{h'}(\cdot|s_{h'}, a_{h'})$. For simplicity, we denote
$\langle P_h(\cdot|s,a), V \rangle = \E_{s' \sim P_h(\cdot|s,a)}[V(s')]$.
We denote the optimal policy for a given MDP with $\pi^*$. 
Similarly, for the optimal Q-function and V-function 
we use
$Q^*_h(s,a) = Q^{\pi^*}_h(s,a), V^*_h(s) = V^{\pi^*}_h(s)$, respectively.

The goal of the learner is to improve her performance as she interacts 
with the unknown environment.
In the episodic setting, the agent commits to a policy at the beginning
of every episode.
We let $\pi^k$ denote the policy that the agent uses in the $k$-th episode.
At each step $h \in [H]$, the agent observes the state $s_h^k$, 
chooses an action according to $\pi^k$, and then observes the reward
$r_h(s_h^k,a_h^k)$ and the next state $s^k_{h+1} \sim P_h(\cdot|s_h^k,a_h^k)$.
 In this work, to measure the performance
of the agent we use the notion of \emph{regret},
defined as
\begin{align*}
    \reg(K) = \sum_{k=1}^K\left(V_1^*(s_1) - V_1^{\pi^k}(s_1) \right).
\end{align*}
The regret measures the difference between the total reward that the agent would have accumulated if she was following the optimal policy and the reward she actually accumulates.
The intuition behind striving for algorithms that guarantee sub-linear regret
is that we want, as $K \rightarrow \infty$, the average 
reward of the learner who follows the algorithm to approach
that of the optimal policy.

We now describe the tabular MDP setting, which is arguably the simplest
setting one can work on. The lower bounds on the regret that we state
apply to this setting, whereas both of the upper bounds we derive capture 
the tabular
MDP setting. In this regime, there are $|\S|$ states, $|\AA|$ actions, and for 
each step $h \in [H]$ the transition probability is given by $P_h(s'|s,a)$, where $\sum_{s' \in \S} P_h(s'|s,a) = 1$, and
the reward is denoted by $r_h(s,a)$. The caveat with this setting is that in practical applications
the state space and the action space can be very large, so the bounds that depend explicitly
on the cardinality of these quantities are not very useful.

An important assumption we make in order
to achieve logarithmic regret guarantees is that the minimum \emph{sub-optimality} gap is positive.

\begin{definition}\label{def:minimum gap}
We define the sub-optimality gap of a state-action pair $(s,a)$ at step $h$ to be
\begin{align*}
    \gap_h(s,a) = V^*_h(s) - Q^*_h(s,a).
\end{align*}
The minimum sub-optimality gap is defined to be
\begin{align*}
    \gap_{\min} = \min_{h,s,a}\{\gap_h(s,a): \gap_h(s,a) \neq 0 \}.
\end{align*}
\end{definition}
It is well-known that if we do not make any assumptions the best 
regret guarantee we can hope for is $O(\sqrt{T})$ \citep{jaksch2010near}. 
In this work, we derive instance-dependent
regret guarantees that achieve an exponential improvement
on $T$ when $\gap_{\min} > 0$.

\subsection{Model-Free Assumption}
\label{sec:linear-mdp-assumption}
In this paper, we deal with general function classes.  
In the model-free setting we assume that we have access to a function class
$\FF \subseteq \{f:\S \times \AA \rightarrow [0,H+1] \}$.
Our goal is to use the functions in $\FF$ to approximate
the optimal Q-function. In order to derive
meaningful results
we assume that this class has some structure. 
We follow the same assumption as in~\cite{wang2020reinforcement, kong2021online}.

\begin{assumption}[Bellman Operator Assumption]
\label{as:linear-mdp}

For any $h \in [H]$ and $V: \S \rightarrow [0,H]$ there exists some $f_V \in \FF$
such that for all $(s,a) \in \S \times \AA$ we have
\begin{align*}
    f_V(s,a) = r_h(s,a) + \sum_{s' \in \S} P_h(s'|s,a) V(s').
\end{align*}
\end{assumption}

The intuition behind this assumption is that if we apply the one-step Bellman backup operator 
to some value function $V$, i.e.
\begin{align*}
    r_h(s,a) +  \sum_{s' \in \S} P_h(s'|s,a) V(s'),
\end{align*}
the result will remain in the 
function class. Thus, it implicitly poses some constraints both
on the transition probabilities and the reward function.
It is known that both the tabular setting and the
linear MDP setting \citep{yang2019sample, jin2020provably} satisfy this assumption.

Another assumption we make is that the function class and the state-action space
have bounded covering numbers. We will show that the dependence of the regret on the covering number
is poly-logarithmic. This
assumption has also appeared in other works \citep{russo2013eluder, wang2020reinforcement, jin2021bellman, kong2021online}.

\begin{assumption}[Bounded Covering Number]
\label{as:bounded-covering-number}
We say that $\NN(\FF,\varepsilon)$ is a bound on the $\varepsilon$-covering number of $\FF$, if
for any $\varepsilon > 0$ there is an $\varepsilon$-cover $\cC(\FF, \varepsilon) \subseteq \FF$
with size $|\cC(\FF, \varepsilon)| \leq \NN(\FF, \varepsilon)$, so that
for all $f \in \FF$ there is some $f'\in \CC(\FF,\varepsilon)$ such that
$||f-f'||_{\infty} \leq \varepsilon$.
Similarly, we say that $\NN(\S \times \AA,\varepsilon)$ is a 
bound on the $\varepsilon$-covering number of $\S \times \AA$ with
respect to $\FF$, if for any $\varepsilon > 0$ there is
an $\varepsilon$-cover $\cC(\S \times \AA, \varepsilon)
\subseteq \S \times \AA$ with size $|\cC(\S\times \AA,
\varepsilon)| \leq \NN(\S \times \AA, \varepsilon)$, so that
for all $(s,a) \in \S \times \AA$ there is some $(s',a')\in \CC(\S\times\AA,\varepsilon)$ such that
$\sup_{f \in \FF}|f(s,a)-f(s',a')| \leq \varepsilon$.
\end{assumption}
The intuition behind this assumption is straightforward: even
if the function class or the state-action space are infinite,
we can approximate them using a small number of points. 

\subsection{Model-Based Assumption}
The assumptions in Section~\ref{sec:linear-mdp-assumption} are \emph{model-free} since they impose some structure on the function class
that approximates the Q-function instead of the transition kernel.
In order to derive our results, we can also
follow a different route and impose some structure directly
on the transition kernel~\citep{ayoub2020model}.

\begin{assumption}[Known Transition Model Family]
\label{as:model-based-assumption}
For all $h \in [H]$, the transition model $P_h$ belongs to a family of
models $\PP_h$ which is known to the learner. The elements of $\PP_h$
are transition kernels that map state-action pairs to signed distributions
over the state space $\S$.
\end{assumption}

We allow signed distributions in our model class to increase
its generality. For example, this is useful when we are given
access to a model class that can be compactly represented
only when it includes non-probability kernels. For an extensive
discussion about this, the reader is referred to~\cite{pires2016policy}.

Transition kernels, either parametric or non-parametric, have been
used to model complex stochastic controlled systems. For instance,
transitions in robotics systems are often modelled using
parameters of the environment, such as friction.

An important class that satisfies this assumption are the 
\emph{linear mixture models}.

\begin{definition}
\label{def:linear-mixture-mdp}
The class of models $\PP$ with
feature mapping $\phi(s'|s,a): \S \times \AA \times \AA \rightarrow \R^d$ 
and some $\theta^* \in \R^d, ||\theta^*||_2\leq C_{\theta},$ is called linear mixture model if:
\begin{itemize}
    \item   $ P(s'|s,a) = \left\langle\phi(\cdot|s,a), \theta^* \right\rangle.$
    \item For any bounded function $V: \S \rightarrow [0,H]$ and any pair $(s,a) \in \S \times \AA$, we have $||\phi_V(s,a)||_2 \leq \sqrt{H}$, where
    $$\phi_V(s,a) = \left\langle \phi(\cdot| s,a), V\right \rangle.$$
\end{itemize}
\end{definition}

One way to interpret the linear mixture model is as an aggregation of 
some basis models which are known to the designer~\citep{modi2020sample}.
Another interesting way to think about it comes from large-scale
queuing networks where both the arrival rate of jobs and
the processing speed for the queues are unknown. If we approximate
this system in discrete time, then the transition matrix
from timestep $t$ to timestep $t+\Delta t$ approaches that
of a linear function with respect to the arrival rate
and the processing time~\citep{gnedenko1989introduction}.

Another interesting setting that satisfies this assumption
is the linear-factored MDP~\citep{yang2020reinforcement}. 
Assuming that the state space is discrete, we have that
\begin{align*}
    P(s'|s,a) &= \phi(s,a)^T M \psi(s')\\
    &= \sum_{i=1}^{d_1}\sum_{j=1}^{d_2} M_{ij}[\psi_j(s')\phi_i(s,a)].
\end{align*}

\subsection{Complexity Measure: Eluder Dimension}
\label{sec:complexity-measure}

Our results depend on the complexity of the function classes and the model
classes that we consider.
To measure this complexity, we use the \emph{eluder dimension} of 
these classes~\citep{russo2013eluder}.

\begin{definition}
\label{def:eluder-dimension}
Fix some $\varepsilon \geq 0$ and a sequence of $n$ points $\mathcal{Z} = \{(x_i) \}_{i \in [n]} \subseteq \XX$. Then:
\begin{enumerate}
    \item A point $x \in \XX$ is $\varepsilon$-dependent on $\mathcal{Z}$ with respect to $\mathcal{F}$ if for all $f, f' \in \mathcal{F}$ such that $||f-f'||_{\mathcal{Z}} \leq \varepsilon$ it holds that $|f(x) - f'(x)| \leq \varepsilon$.
    
    \item A point $x$ is $\varepsilon$-independent of $\mathcal{Z}$ with respect to $\mathcal{F}$ if $x$ is not $\varepsilon$-dependent on $\mathcal{Z}$.
    
    \item The $\varepsilon$-eluder dimension of $\mathcal{F}$, which is denoted by $\text{dim}_E(\mathcal{F}, \varepsilon)$, is the length of the longest sequence of elements in $\XX$ such that every element in this sequence is $\varepsilon'$-independent of its predecessors, for some $\varepsilon' \geq \varepsilon$.
\end{enumerate}

Intuitively, the eluder dimension of $\FF$ quantifies the
largest set of elements $\ZZ \subseteq \XX$ so that if all $f \in \FF$
are close with respect to $\ZZ$, then they are close
on all elements of $\XX$.

 It is known that when $\XX = \S \times \AA, f: \S \times \AA
\rightarrow [0,H]$, and $\S, \AA$ are finite, we have that $\dim_{E}(\FF,
\varepsilon) \leq |\S| \cdot |\AA|$, for all $\varepsilon > 0$~\citep{russo2013eluder, wang2020reinforcement}. Moreover,
when $\FF$ is the class of linear functions, i.e., $f_{\theta}(s,a) = \theta^T \phi(s,a)$, for a given feauture vector $\phi(s,a)$, the eluder dimension of
$\FF$ is bounded by $\dim_E(\FF, \varepsilon) = O(d \log(1/\varepsilon))$, for all $\varepsilon > 0$. 
We also remark that more classes including generalized
linear functions and bounded degree polynomials
have bounded eluder dimension~\citep{russo2013eluder, osband2014model,li2021eluder}.


\end{definition}



\subsection{Switching Cost}
\label{sec:switching-cost}
Essentially,
the \emph{switching cost} or the \emph{adaptivity complexity} measures the number of episodes the algorithm
has to update its policy in order to achieve the guaranteed regret bound~\citep{bai2019provably, kong2021online}. More formally:

\begin{definition}
\label{def:switching-cost}
We define the switching cost of an algorithm $A$ over $K$ episodes to be
\begin{align*}
    N_{\textnormal{switch}} = \sum_{k=1}^{K-1} \mathbbm{1}[\pi_{k} \neq \pi_{k+1}].
\end{align*}
\end{definition}
\section{Overview of the Algorithms and     Main Results}
\label{sec:overview-of-main-results}

In this section, we present our main results and
give a high-level description of the techniques
we use. We treat both the model-free and the model-based setting in a unified way.
Our algorithm, inspired by~\cite{kong2021online}, is presented
in Algorithm~\ref{alg:main algorithm}. The only differences
between the two settings are the different sampling routine
and Q-function estimator used by Algorithm~\ref{alg:main algorithm}. In a nutshell, our low-switching cost algorithm works as follows:
\begin{itemize}
    \item After each round of the interaction with the MDP, 
    Algorithm~\ref{alg:main algorithm} adds elements to the current dataset
    with some probability that depends
    on their significance 
    and updates the policy only if the dataset has changed.
    This guarantees that the adaptivity of the algorithm
    depends logarithmically on $T$, without hurting the
    regret guarantee.
    
    \item We use a least-squares estimate of the Q-function
    (transition kernel) in the model-free (model-based) setting.
 
    \item We add a bonus to this estimate which encourages exploration
    and, with high probability, guarantees that the current estimate of
    the Q-function serves as an element-wise upper bound of $Q^*$. This
    bonus is based on a sub-sampled dataset that we have accumulated
    from previous interactions with the MDP.
\end{itemize}


Before we delve deeper into the two settings separately,
we describe a parameter that is crucial for both of the algorithms
we are using.
Following~\cite{kong2021online},
we define the \emph{sensitivity} of an element $z$ with respect to
a dataset $\ZZ$ and a function class $\FF$ to be
\begin{align*}
     \textnormal{sensitivity}_{\ZZ,\FF}(z) =
     \min\left\{\sup_{f_1, f_2 \in \FF}\frac{\left(f_1(z) - f_2(z) \right)^2}{\min\{||f_1 - f_2||_{\ZZ}^2, T(H+1)^2 \} + \beta} \right\}.
\end{align*}

Intuitively, this parameter captures the importance
of the current element $z$ relative to the dataset we are working 
with. 
We will elaborate on the choice of the parameter $\beta$
for each of the two settings separately.
To establish the regret guarantee,
we propose a novel regret decomposition where we utilize
the ``peeling technique" that has been applied in prior
works in 
local Rademacher complexities~\citep{bartlett2005local}
and in RL~\citep{he2021logarithmic, yang2021q}. 
By doing that, we show how the regret of the algorithm
relates to the suboptimality gap.

To establish the lower bound on the regret of any algorithm in the settings
we are interested in, we utilize a result that was proved in~\cite{ok2018exploration, simchowitz2019non}. It states that for all algorithms
that achieve sublinear regret, there exists a tabular MDP
where its regret is at least $\Omega(\poly(\log(T))\cdot \poly(H) \cdot 1/\gap_{\min})$.



\subsection{Model-Free Setting}
\label{sec:overview-model-free}

We first present the approach we use in the model-free setting, i.e., where we have access to some 
function class $\FF$  and state-action space $\S \times \AA$ that
satisfy Assumption~\ref{as:linear-mdp} and
Assumption~\ref{as:bounded-covering-number}, respectively. The Q-function estimator and the sampling routine 
that we use for this setting are presented in Algorithm~\ref{alg:model-free q estimator} and Algorithm~\ref{alg:model-free sampling routine}, respectively.
The dataset includes pairs of the form $z_h^k = (s_h^k,a_h^k)$.
The Q-function routine is a least-squares 
estimator that is based on all the previous interactions with the MDP. Notice that the bonus function is based only on the sub-sampled
dataset and depends on a hardcoded parameter $\beta$. This parameter is chosen in a way that ensures
the Q-function is an optimistic estimate of the actual one
and the bonus we add is not too large. Since we are working
with the same assumptions as~\citet{kong2021online},
our choice of $\beta$ coincides with theirs.

A crucial part of the algorithm is the online sub-sampling routine.
The reason we are using this routine is twofold. 
Firstly, if we use the entire dataset there will be a huge 
number of distinct elements in it, which can make the exploration
bonus unstable since it changes constantly and can take
infinitely many different values. In order to establish
the optimism of the Q-function estimation, namely, 
\begin{align*}
    Q_h^*(s,a) \leq Q_h^k(s,a) \leq \langle P_h(\cdot|s,a), V_h^k \rangle + 2b_h^k(s,a),
\end{align*}
\citet{kong2021online} show that it is crucial to bound the complexity of the exploration bonus.
Secondly, if we sub-sample the dataset based on the
importance of the elements, we can achieve the regret guarantees that
we are aiming for by switching the policy only when an important element has been added.
Notice that whenever an element is added to the dataset,
multiple copies are included.
This is to make the
sub-sampled dataset behave like an unbiased estimator of the orignal one.
Then, using concentration bounds one can show that it 
approximates the original one with high probability.
The full description of this procedure is presented in Algorithm~\ref{alg:model-free sampling routine}.
For a more detailed discussion about the importance of sub-sampling
the dataset, the interested reader is referred to~\cite{wang2020reinforcement, kong2021online}.

\begin{algorithm}[ht]
\begin{algorithmic}[1]
\REQUIRE Failure probability $\delta \in (0,1)$, number of episodes $K$, and setting of operation
\STATE $\Tilde{k} \leftarrow 1$
\STATE $\hat{Z}_h^1 \leftarrow \emptyset, \forall h \in [H]$
    

    \FOR{$k \in [K]$}
        \FOR{$h = H, H-1, \ldots, 1$}
            \IF{$k \geq 2$} 
                \STATE $\hat{\ZZ}_h^k \leftarrow \textbf{Online-Sample}(\FF, \hat{\ZZ}_h^{k-1}, z_h^{k-1}, \delta) $
            \ENDIF
        \ENDFOR
        
        \IF{$k= 1 \text{ or } \exists h \in [H]: \hat{Z}_h^k \neq \hat{Z}_h^{\Tilde{k}}$}
            \STATE $\Tilde{k} \leftarrow k$
            \STATE $Q_{H+1}^{k}(\cdot,\cdot) \leftarrow 0,
                    V_{H+1}^{k}(\cdot) \leftarrow 0$
            \FOR{$h=H, H-1, \ldots, 1$}
                \STATE $\mathcal{T}_h^k \leftarrow $ history of execution
                \STATE $Q_h^k(\cdot, \cdot) \leftarrow \textbf{Q-Estimator}(\mathcal{T}_h^k, \ZZ_h^k)$
                \STATE $ V_h^k(\cdot) = \max_{a \in \AA} Q_h^k(\cdot, a)$
                \STATE $\pi_h^k(\cdot) \leftarrow \arg\max_{a \in \AA}Q_h^k(\cdot, a)$
            \ENDFOR
        \ENDIF
        
        \STATE Receive initial state $s_1$ of episode $k$
        \FOR{$h \in [H]$}
            \STATE Take action $a_h^k \leftarrow \pi_h^{\Tilde{k}}(s_h^k)$
        \ENDFOR
    \ENDFOR
\end{algorithmic}
\caption{{\sf Low Switching Cost Value Iteration (with parameters $\delta, K$)}}
\label{alg:main algorithm}
\end{algorithm}

\begin{algorithm}[ht]
\begin{algorithmic}[1]
\REQUIRE Current sub-sampled dataset $\hat{\ZZ}$, history of execution $\mathcal{T}$ 
\STATE $\cD_h^k \leftarrow \{(s_h^{\tau}, a_h^{\tau},
                r_h^{\tau} + V_{h+1}^{k}(s_{h+1}^{\tau}))
                \}_{\tau \in [k-1]}$
\STATE $\hat{f} \leftarrow \arg\min_{f\in \FF} ||f||^2_{\cD}$
\STATE $\hat{F}_h^k \leftarrow \{f_1, f_2 \in \FF: \min\{\|f_1 - f_2\|^2_{\hat{\ZZ}_h^k}, T(H+1)^2  \leq \beta\}$
\STATE $b_h^k(\cdot,\cdot) \leftarrow \sup_{f_1, f_2 \in \hat{\FF}^k_h} |f_1(\cdot,\cdot) - f_2(\cdot, \cdot)|$
\STATE \textbf{Return} $\min\{f_h^k(\cdot,\cdot) + b_h^k(\cdot, \cdot), H\}$
         
\end{algorithmic}
\caption{{\sf Q-function Model-Free Estimator}}
\label{alg:model-free q estimator}
\end{algorithm}

\begin{algorithm}[ht]
\begin{algorithmic}[1]
\REQUIRE Function class $\FF$, current sub-sampled dataset $\hat{\ZZ}$, new element $z = (s,a)$, failure probability $\delta \in (0,1)$ 
\STATE Let $p_z$ be the smallest number such that $1/p_z$ is an integer and $p_z$ is greater than
\begin{align*}
      \min\{1, C \textnormal{sensitivity}_{\hat{\ZZ}, \FF}\newline
         \log(T\NN(\FF, 
        \sqrt{\delta/(64T^3)})/\delta)\}
\end{align*}
\STATE Let $\hat{z} \in \CC{(\S\times \AA, 1/16\sqrt{64T^3/\delta})}$ such that
            $$\sup_{f\in\FF} |f(z) - f(\hat{z})| \leq 1/16\sqrt{64T^3/\delta}$$
\STATE Add $1/p_z$ copies of $\hat{z}$ into $\hat{\ZZ}$ with probability $p_z$

\STATE \textbf{Return} $\hat{\ZZ}$
         
\end{algorithmic}
\caption{{\sf Model-Free Sampling Routine}}
\label{alg:model-free sampling routine}
\end{algorithm}

We are now ready to state our main result in this setting.

\begin{theorem}\label{thm:model-free main result}
There exists an absolute constant $C > 0,$ and a proper parameter $\beta$ for Algorithm~\ref{alg:main algorithm} such that with probability of at least $1- \lceil \log T \rceil e^{-\tau} -\delta$ the regret of the algorithm is bounded by
\begin{gather*}
    \reg(K) \leq
    \frac{C  d_{\FF}  H^5  \log^4T}{\gap_{\min}} + \frac{16 H^2 \tau}{3} + 2,
\end{gather*}
for any $\delta, \tau > 0$, where $d_{\FF} = \dim^2_E(\FF, 1/T)\cdot \log(\NN(\FF,\delta/T^2)/\delta) \cdot  \log(\CC(\S \times \AA, \delta/T^2)/\delta)$ is a parameter that captures
the complexity of the function class. The value of the parameter $\beta$ is 
\begin{gather*}
 \beta = C d'_{\FF} H^2 \log^4T,
\end{gather*}
where $d'_{\FF} = \log(\NN(\FF, \delta/T^3)/\delta)\dim_E(\FF, 1/T)  \log ( \NN(\S \times \AA, \delta/ T^3)/\delta))$.
Moreover, the number of switching policies is bounded by
\begin{align*}
  O\left(H\log(T\NN(\FF,\sqrt{\delta}/T^2)/\delta) \dim_E(\FF,1/T) \log^2T\right).  
\end{align*}

\end{theorem}

Notice that the regret in the previous bound depends on $\log T$ and  $1/\gap_{\min}$. We show
that this dependence is necessary.

\begin{remark}
\label{rem:model-free lower bound on regret}
    There exists an MDP that satisfies Assumption~\ref{as:linear-mdp} and
    Assumption~\ref{as:bounded-covering-number} such that the expected regret 
    of every algorithm is lower bounded by
    $$\E[\reg(K)] \geq \Omega\left(\poly(\log(T)) \cdot \poly(H)\cdot \frac{1}{\gap_{\min}} \right).$$
\end{remark}

It is known that the setting we are studying captures the tabular MDP setting (see e.g. \citep{wang2020reinforcement}).
Hence, we can see that the result by~\cite{ok2018exploration, simchowitz2019non} we mentioned
proves our claim. 

\subsection{Model-Based Setting}
\label{sec:overview-model-based}

In this regime, we assume that the MDP satisfies 
Assumption~\ref{as:model-based-assumption}. We also
assume that the reward function is known to the learner similar
to~\cite{ayoub2020model}. If the reward is unknown, we just
estimate it and construct a confidence region.

Before we discuss the details of our approach, we need to describe an important set of functions that show up in our algorithm
and in the regret guarantee.
Let $\VV$ be the set of all measurable functions that are bounded by $H$. 
Let $\PP_h$ be the set of potential models in step $H$.
We also let $f : \S \times \AA \times \VV \rightarrow \R$ and define the following 
set:
\begin{align*}
  \FF_h &  = \biggl\{f: \exists \Tilde{P}_h \in \PP_h \text{ so that } f(s, a, V)  = \int_\S \Tilde{P}_h(s'| s,a) V(s') ds', \forall (s, a, V) \in \S \times \AA \times \VV \biggr\}.
\end{align*}
The bounds we state scale with the complexity of $\FF_h$.
The Q-function estimator and the sampling routine 
we use that are specific to this setting are presented in Algorithm~\ref{alg:model-based q estimator} and Algorithm~\ref{alg:model-based sampling routine}, respectively. 
The dataset includes elements of the form $z_h^k = (s_h^k,a_h^k,V^k_{h+1}(\cdot))$.
The Q-function routine works in the following way.
We first estimate a model $\hat{P}_h \in \PP_h$ using a least-squares 
estimator that is based on all the previous interactions with the MDP.
 Using a concentration argument for this estimator of the model,
similar to~\cite{russo2013eluder, ayoub2020model},
we can show that for an appropriate choice of $\beta$, the estimated model lies
in a data-dependent ball centered at $\hat{P}_h$, with high probability (see Lemma~\ref{lem:helper lemma from prior work, concentration of LSE} in the Appendix). Thus, 
we can set the bonus function to be the diameter of this ball in order
to ensure that $\hat{Q}_h$ is an optimistic estimate of $Q_h^*$.
In addition, the choice of $\beta$ ensures
that the bonus we add is not very large. Notice also
that since the concentration argument in this setting
differs with that in the model-free setting, we do not
need to round the elements that we are adding to the 
sub-sampled dataset.

In this setting, the main reason
we sub-sample the dataset is to achieve logarithmic adaptivity.
Moreover, not using the entire dataset to compute the bonus function
improves the computational complexity of our algorithm.
To bound the adaptivity complexity, we use a similar approach as
in~\cite{kong2021online}.

\begin{algorithm}[ht]
\begin{algorithmic}[1]
\REQUIRE Function class $\FF$, current sub-sampled dataset $\hat{\ZZ}$, current regression dataset $\cD$ 
 \STATE $\hat{P}_h^k \leftarrow \arg\min_{P \in \PP_h} \newline
                    \hspace*{0.1em}\sum_{k'=1}^{k-1} \left(\langle P(\cdot|s_h^{k'}, a_h^{k'}), V_{h+1}^{k'}\rangle -  V_{h+1}^{k'}(s_{h+1}^{k'}) \right)^2$
                
\STATE $\FF_h^k = \{f_1, f_2:  \min\{||f_1 - f_2||_{\hat{\ZZ}_h^k}, T(H+1)^2 \} \leq \beta\}$
\STATE 
$b_h^k(\cdot,\cdot) \leftarrow \sup_{f_1, f_2 \in \FF^k_h}
|f_1(\cdot,\cdot, V_{h+1}) - f_2(\cdot, \cdot, V_{h+1})|$

\STATE \textbf{Return} $\min\{r_h(\cdot, \cdot) +  \langle \hat{P}^k_h(\cdot|\cdot, \cdot), V^k_{h+1}\rangle + b_h^k(\cdot,\cdot), H\}$
         
\end{algorithmic}
\caption{{\sf Q-function Model-Based Estimator}}
\label{alg:model-based q estimator}
\end{algorithm}

\begin{algorithm}[ht]
\begin{algorithmic}[1]
\REQUIRE Function class $\FF$, current sub-sampled dataset $\hat{\ZZ}$, new element $z = (s,a,V)$, failure probability $\delta \in (0,1)$ 
\STATE Let $p_z$ be the smallest number such that $1/p_z$ is an integer and 
    $p_z$ is greater than
        $$\min\{1, C\cdot \textnormal{sensitivity}_{\hat{\ZZ}, \FF} \cdot \log(T\NN(\FF, 
        \sqrt{\delta/(64T^3)})/\delta)\}$$
\STATE Add $1/p_z$ copies of $z$ into $\hat{\ZZ}$ with probability $p_z$

\STATE \textbf{Return} $\hat{\ZZ}$
         
\end{algorithmic}
\caption{{\sf Model-Based Sampling Routine}}
\label{alg:model-based sampling routine}
\end{algorithm}

We are now ready to state our main result in this setting.
\begin{theorem}\label{thm:model-based main result}
There exists an absolute constant $C$ and a proper value of the parameter $\beta$ for Algorithm~\ref{alg:main algorithm} such that with probability at least $1- \lceil \log T \rceil e^{-\tau} - \delta$ the regret of the algorithm is bounded by
\begin{align*}
    \textnormal{Regret}(K) \leq \frac{C d_{\FF} H^5 \log T  }{ \gap_{\min}} + \frac{16 H^2 \tau}{3} + 2,
\end{align*}
where $d_{\FF} = \log(\NN(\FF,1/T)/\delta) \dim^2_E(\FF, 1/T)$.
The value of the parameter $\beta$ is 
\begin{align*}
   \beta & = 4H^2\log(2\NN(\FF,1/T)/\delta)  + 4/H\left(C + \sqrt{H^2/4\log(T/\delta)}\right), 
\end{align*}
where $\NN(\FF,1/T) = \max_{h \in [H]} \NN(\FF_h,1/T), 
\dim_E(\FF, 1/T) = \max_{h \in [H]} \dim_E(\FF_h,1/T)$.

The number of switching policies is bounded by
$$O\left(H\log(T\NN(F,\sqrt{\delta/(64T^3)})/\delta)\dim_E(\FF,1/T)\log^2T\right).$$

\end{theorem}

As in Section~\ref{sec:overview-model-free}, notice that the regret in the 
previous bound depends on $\log T$ and $1/\gap_{\min}$. We show
that this dependence is indeed necessary.

\begin{remark}
\label{rem:model-based lower bound on regret}
    There exists an MDP that satisfies Assumption~\ref{as:model-based-assumption}  
    such that the expected regret 
    of every algorithm is lower bounded by
    $$\E[\reg(K)] \geq \Omega\left(\poly(\log(T))  \poly(H) \frac{1}{\gap_{\min}} \right).$$
\end{remark}

It is known that the setting we are studying captures the tabular MDP setting (see 
e.g.~\citep{ayoub2020model}), so the proof follows from 
the result by~\cite{ok2018exploration, simchowitz2019non}.

We remark that if we want to derive instance-independent regret guarantees
and follow the regret decomposition in~\cite{ayoub2020model},
we can get a $\sqrt{T}$-regret with logarithmic adaptivity.




  
  
\section{Proof Sketch of Main Results}
\label{sec:proof sketch of main results}

In this section, we sketch the proof of our main results.
Due to space limitation, we only discuss the model-free setting.
The full proofs in the model-free, model-based setting can be found
in Appendix~\ref{sec:linear-mdp-proof}, 
Appendix~\ref{sec:model-based proof}, respectively.

The first step in our analysis is the regret decomposition of the algorithm.
Lemma~\ref{lem:model-free-regret-decomposition} shows that
$\E\left[\reg(K)\right] = \E\left[\sum_{k=1}^K \sum_{h=1}^H \gap_h(s_h^k,a_h^k)\right]$. Thus, we see that to bound
the regret it is enough to bound $\sum_{k=1}^K\gap_h(s_h^k,a_h^k)$ for every $h \in [H]$.
Towards this end, notice that $\gap_h(s_h^k,a_h^k) \in [\gap_{\min}, H]$. We apply the ``peeling technique'' that
has also been used in local
Rademacher complexities \citep{bartlett2005local} and in \cite{he2021logarithmic, yang2021q}. The idea is to split the interval $[0,H]$
into $\log(H/\gap_{\min})$ intervals, where the $i$-th interval is
$[2^{i-1}\gap_{\min}, 2^i\gap_{\min}]$. Hence, for every $\gap_h(s_h, a_h)$
that falls in the $i$-th interval its contribution 
to the regret is at most $2^{i}\gap_{\min}$.
 Thus, to bound the regret it suffices
to bound the number of suboptimalities that fall into every interval.
Notice that for some $\gap_h(s_h^k, a_h^k)$ in this interval we have that 
\begin{align*}
   V_h^*(s_h^k) - Q_h^{\pi_k}(s_h^k, a_h^k) \geq \gap_h(s_h^k, a_h^k) \geq 2^{i-1}\gap_{\min}, 
\end{align*}
so it suffices to bound the number of sub-optimalities $V_h^*(s_h^k) - Q_h^{\pi_k}(s_h^k, a_h^k)$
that fall into the $i$-th interval.
Both for the model-free and the model-based setting, we
can derive such a bound.
Finally, notice that once we have bounded the number of suboptimilaties
in every interval, it is not difficult to bound the total regret. Let
$C_i = [2^{i-1}\gap_{\min}, 2^i\gap_{\min})$ and $N = \log(H/\gap_{\min})$. 
Then, we know that $\reg(K)$ can be upper bounded by
\begin{align*}
    & \sum_{k=1}^K \sum_{h=1}^H \gap_h(s_h^k,a_h^k) \\
    & \qquad =
     \sum_{i=1}^{N}\sum_{\gap_h(s_h^k,a_h^k) \in 
    C_i} \gap_h(s_h^k,a_h^k) \\
    & \qquad \leq
    \sum_{i=1}^{N}\sum_{k=1}^K 2^i \mathbbm{1}\left[\gap_h(s_h^k,a_h^k) \in C_i\right]\\
   & \qquad \leq\sum_{i=1}^{N}2^i\sum_{k=1}^K  \mathbbm{1}\left[V^*_h(s_h^k) - Q_h^{\pi_k}(s_h^k,a_h^k) \geq 2^{i-1}\gap_{\min}\right].
 \end{align*}
Hence, deriving the regret guarantee boils down to bounding
$\sum_{k=1}^K\mathbbm{1}\left[V^*_h(s_h^k) - Q_h^{\pi_k}(s_h^k,a_h^k) \geq 2^{i}\gap_{\min}\right]$. We provide such a bound in~Lemma~\ref{lem:indicator bound on every interval},
which depends polynomially on $\log T, 1/\gap_{\min}$,
and the complexity parameters of the function class.
The outline of the proof is the following. We fix some
episode $h \in [H]$ and let $K'$ be
the set of rounds where $V^*_h(s_h^k) - Q_h^{\pi_k}(s_h^k,a_h^k) \geq 2^{i}\gap_{\min}$.
To get a bound on $|K'|$, our approach is to lower bound and upper bound
the quantity $\sum_{i=1}^{|K'|}Q_h^{k_i}(s_h^{k_i}, a_h^{k_i}) - Q_h^{\pi_{k_i}}(s_h^{\pi_{k_i}}, a_h^{\pi_{k_i}})$ by $f_1(|K'|),
f_2(|K'|)$, respectively. Then, we use the fact that
$f_{1}(|K'|) \leq f_{2}(|K'|)$ to establish our bound.
For the lower bound, using the definition of $K'$ we get 
$f_1(|K'|) = 2^i\gap_{\min}|K'|$. For the upper bound,
we leverage the fact that $Q_h^{k_i}(s_h^{k_i}, a_h^{k_i}) \leq
\langle P_h(\cdot|s_h^{k_i}, a_h^{k_i}), V_{h+1}^{k_i} \rangle +
2b_h^{k_i}(s_h^{k_i}, a_h^{k_i})$ (cf.~Lemma~\ref{lem:single step
bellman error}) and obtain $f_2(|K'|) \leq
\sum_{i=1}^{|K'|}\sum_{h'=h}^H \varepsilon_{h'}^{k_i} +
\sum_{i=1}^{|K'|}\sum_{h'=h}^H
b_{h'}^{k_i}(s_{h'}^{k_i},a_{h'}^{k_i})$, where
$\varepsilon_h^{k_i}$ forms a bounded martingale difference sequence.
We bound each term on the RHS separately. 
For the first one, we use the Azuma-Hoeffding inequality which
can be found in Lemma~\ref{lem:azuma-hoeffding inequality}.
For the second term, we generalize the bound
on the summation of the bonus functions over all the episodes
from~\cite{kong2021online}, and show that a similar bound
holds for the summation of the bonus over \emph{any} set
of episodes $K'$ (see Lemma~\ref{lem:bound of bonus functions over all rounds}). 
Putting everything together, we get that 
$|K'| = O\left(1/(4^{i}\gap_{\min})\cdot H^4 \cdot \log^4 T \cdot \poly(d_{\FF})\right)$,
where $d_{\FF}$ is the complexity parameter of the class.

\section{Conclusion}
\label{sec:conclusion}
In this paper, we consider episodic RL with general
function approximation.  We prove
that there are algorithms with logarithmic adaptivity complexity
both in the model-free and model-based
settings that achieve logarithmic instance-dependent
regret guarantees.
An interesting open question
is to establish the optimal dependence of the regret
guarantees on the planning horizon $H$.

\bibliographystyle{agsm}
\bibliography{bib}

\newpage
\appendix
\onecolumn
\section{Proof of Theorem~\ref{thm:model-free main result}}
\label{sec:linear-mdp-proof}

In this section, our main goal is to prove
Theorem~\ref{thm:model-free main result}. Recall that
we assume we have access to a set
$\FF \subseteq\{f:\S\times \AA \rightarrow [0,H+1]\}$,
which we use to approximate the Q-function. For this
set, we work with Assumption~\ref{as:linear-mdp} and
Assumption~\ref{as:bounded-covering-number}.

The proofs of the supporting lemmas are postponed to Appendix~\ref{sec:supporting lemmas for model-free theorem}.
Before we are ready to prove our result, we need to 
discuss some results of prior works that are crucial
to our proof.

The regret decomposition in~\cite{he2021logarithmic}, gives us that
\begin{lemma}\citep{he2021logarithmic}
\label{lem:model-free-regret-decomposition}
For any MDP $M$ we have that
\begin{align*}
    \E\left[\reg(K) \right] = \E\left[ \sum_{k=1}^K \sum_{h=1}^H \gap_h(s_h^k, a_h^k)\right].
\end{align*}

Moreover, for any $\tau > 0$ it holds with probability at least $1-\lceil \log T \rceil e^{-\tau}$ that
\begin{align*}
    \reg(K) \leq 2\sum_{k=1}^K \sum_{h=1}^H \gap_h(s_h^k, a_h^k) + \frac{16H^2\tau}{3} + 2.
\end{align*}
\end{lemma}



The following lemma resembles Lemma~6.3~\citep{he2021logarithmic}. Its proof is 
postponed to Appendix~\ref{sec:supporting lemmas for model-free theorem}.
\begin{lemma}\label{lem:sum of suboptimalities over all K for some h}
If we pick $$\beta = CH^2 \log(T\NN(\FF, \delta/T^2)/\delta)\dim_E(\FF, 1/T) \log^2T \log \left( \CC(\S \times \AA, \delta/ T^2)T/\delta)\right),$$
for come constant C and for $h \in [H]$, then we have that with probability at least $1-2K\delta$

\begin{align*}
    \sum_{k=1}^K \left( V_h^*(s_h^k) - Q_h^*(s_h^k, a_h^k) \right) \leq \frac{C H^4 \dim^2_E(\FF, 1/T) \log^2T \log(T\NN(\FF,\delta/T^2)/\delta)\log(\CC(\S \times \AA, \delta/T^2)T/\delta)}{\gap_{\min}}.
\end{align*}

\end{lemma}

We are now ready to state the regret bound of our algorithm. In particular
the regret guarantee follows from the regret decomposition and the bound
we established before.
\begin{lemma}\label{lem:model-free regret guarantee}
There exists a constant $C$ and proper values of the parameter $\beta$ of Algorithm~\ref{alg:main algorithm} such that with probability at least $1- \lceil \log T \rceil e^{-\tau} -2K\delta$ the regret of the algorithm is bounded by

\begin{align*}
    \reg(K) \leq \frac{C H^5 \dim^2_E(\FF, 1/T) \log^2T \log(T\NN(\FF,\delta/T^2)/\delta)\log(\CC(\S \times \AA, \delta/T^2)T/\delta)}{\gap_{\min}}
+ \frac{16 H^2 \tau}{3} + 2.
\end{align*}

The choice of the parameter is $$\beta = CH^2 \log(T\NN(\FF, \delta/T^2)/\delta)\dim_E(\FF, 1/T) \log^2T \log \left( \CC(\S \times \AA, \delta/ T^2)T/\delta)\right).$$
\end{lemma}

\begin{proof}
Throughout the proof, we condition on the events described in Lemma~\ref{lem:model-free-regret-decomposition},~\ref{lem:sum of suboptimalities over all K for some h} which happen with probability at least $1- \lceil \log T \rceil e^{-\tau} -2K\delta$.

From Lemma~\ref{lem:model-free-regret-decomposition} we have that

\begin{align*}
    \reg(K) \leq 2 \sum_{k=1}^K\sum_{h=1}^H \gap_h(s_h^k, a_h^k) + \frac{16 H^2 \tau}{3} + 2.
\end{align*}

We can bound the first term on the RHS using Lemma~\ref{lem:sum of suboptimalities over all K for some h} as follows

\begin{align*}
    2 \sum_{k=1}^K\sum_{h=1}^H \gap_h(s_h^k, a_h^k) &= 2 \sum_{k=1}^K\sum_{h=1}^H  \left( V_h^*(s_h^k) - Q_h^*(s_h^k, a_h^k) \right) \\&\leq
    \frac{C H^5 \dim^2_E(\FF, 1/T) \log^2T \log(T\NN(\FF,\delta/T^2)/\delta)\log(\CC(\S \times \AA, \delta/T^2)T/\delta)}{\gap_{\min}}.
\end{align*}

Hence, for the total regret we have that

\begin{align*}
    \reg(K) \leq \frac{C H^5 \dim^2_E(\FF, 1/T) \log^2T \log(T\NN(\FF,\delta/T^2)/\delta)\log(\CC(\S \times \AA, \delta/T^2)T/\delta)}{\gap_{\min}} + \frac{16 H^2 \tau}{3} + 2.
\end{align*}

\end{proof}

Finally, we state the bound on the adaptivity of our algorithm. In particular,
since our algorithm is the same as in~\cite{kong2021online},
the logarithmic switching cost follows directly from their result.

\begin{lemma}\citep{kong2021online}
\label{lem:linear-mdp-adaptivity}
For any fixed $h \in [H]$, With probability $1-\delta$, the sub-sampled dataset $\hat{\ZZ}_h^k$ changes at most
$$O\left(\log(T\NN(\FF,\sqrt{\delta/T^3})/\delta) \dim_E(\FF,1/T) \log^2T\right)$$ times.
\end{lemma}

We are now ready to prove Theorem~\ref{thm:model-free main result}.

\begin{prevproof}{Theorem}{thm:model-free main result}

The proof of this theorem follows by using Lemma~\ref{lem:model-free regret guarantee}
and taking a union on the result of Lemma~\ref{lem:linear-mdp-adaptivity}
and setting the error probability accordingly.
\end{prevproof}

\subsection{Supporting Lemmas: Theorem~\ref{thm:model-free main result}}
\label{sec:supporting lemmas for model-free theorem}

In this section, we present the proof of Lemma~\ref{lem:sum of suboptimalities over all K for some h}.

First, we need to show that the sub-sampled dataset is a good approximation
of the original one. To this end, we use Proposition~1 from~\cite{kong2021online}.

\begin{proposition}\citep{kong2021online}\label{prop:sub-sampled dataset is good approximation of the original}
For any $h, k \in [H] \times [K]$ we let 
\begin{align*}
    \underline{b}_h^k(\cdot, \cdot) = \sup_{\| f_1 - f_2\| ^2_{\mathcal{Z}^k_h} \leq \beta/100}|f_1(\cdot, \cdot) - f_2(\cdot, \cdot)|,\\
    \overline{b}_h^k(\cdot, \cdot) = \sup_{ \| f_1 - f_2\|^2_{\ZZ^k_h} \leq 100\beta}|f_1(\cdot, \cdot) - f_2(\cdot, \cdot) | .
\end{align*}

Then, with probability at least $1 - \delta/32$ we have that

\begin{align*}
    \underline{b}_h^k(\cdot, \cdot) \leq b_h^k(\cdot, \cdot) \leq \overline{b}_h^k(\cdot, \cdot).
\end{align*}

\end{proposition}




    


We now present a generalized version of Lemma~11~\citep{kong2021online} that will be
used to bound the regret of our algorithm. Essentially, this gives a bound on
the summation of the bonus terms  over a set of episodes $K' \subseteq[K]$ in terms
of the eluder dimension of the function class and the number of episodes. 
\begin{lemma}\label{lem:bound of bonus functions over all rounds}
For any set $K' \subseteq [K]$ , with probability at least $1-\delta/32$ we have that
\begin{align*}
    \sum_{i=1}^{|K'|} \sum_{h=1}^H b_h^{k_i}(s_h^{k_i},a_h^{k_i}) \leq H + H(H+1) \cdot \textnormal{dim}_E(\mathcal{F}, 1/T) + CH\sqrt{\textnormal{dim}_E(\mathcal{F},1/T)|K'|\cdot \beta}, 
\end{align*}
where $C > 0$ is some constant.
\end{lemma}

\begin{proof}
Throughout the proof, we condition on the event defined in Proposition~\ref{prop:sub-sampled dataset is good approximation of the original}. This gives us that for any $k \in K', h \in H$
\begin{align*}
    b_h^k(s_h^k, a_h^k) \leq \Bar{b}_h^k(s_h^k, a_h^k) = \sup_{||f_1 - f_2||^2_{\ZZ_h^k} \leq 100\beta}|f_1(s_h^k, a_h^k) - f_2(s_h^k, a_h^k)|.
\end{align*}

We bound $\sum_{i=1}^{K'}\Bar{b}_h^{k_i}(s_h^{k_i}, a_h^{k_i})$ for each $h \in [H]$ separately.\\

Given some $\epsilon > 0$, we define $\mathcal{L}_h = \{ (s_h^{k_i}, a_h^{k_i}): k_i \in K', \Bar{b}_h^{k_i}(s_h^{k_i}, a_h^{k_i}) > \epsilon\}$, 
i.e. the set of state-action pairs at step $h$ and some episode in $K'$ where
the bonus function has value greater than $\epsilon$. Consider some $k \in K'$. We denote 
$L_h = |\mathcal{L}_h|$, $\tilde{\ZZ}_h^{k} = \{(s_h^k, a_h^k) \in \ZZ_h^k, k \in K' \}$ . 
Our goal is to show that there is some $z_h^k = (s_h^k, a_h^k) \in \mathcal{L}_h$ 
that is $\epsilon$-dependent on at least $L_h/\textnormal{dim}_E(\FF, \epsilon) - 1$ disjoint subsequences in $\ZZ_h^k \cap \mathcal{L}_h$. We also denote $N = L_h/\textnormal{dim}_E(\FF,\epsilon) - 1$.

To do that, we decompose $\mathcal{L}_h$ into $N+1$ disjoint subsets and we denote the $j$-th subset by $\mathcal{L}_h^j$. 
We use the following procedure. Initially we set $\LL_h^j = \emptyset$ for all $j \in [N+1]$
and consider every $z_h^k \in \LL_h$ in a sequential manner. 
For each such $z_h^k$ we find the smallest index $j, 1 \leq j \leq N$, 
such that $z_h^k$ is $\epsilon$-independent of the elements in $\LL_h^j$ with
respect to $\FF$. If there is no such $j$, we set $j=N+1$. 
Then, we update $\LL_h^j \leftarrow \LL_h^j \cup z_h^k$.
Notice that after we go through all the elements of $\LL_h$,
we must have that $\LL_h^{N+1} \neq \emptyset$. 
This is because every set $\LL_h^j, 1 \leq j \leq N$, contains at most 
$\textnormal{dim}_E(\FF, \epsilon)$ elements. 
Moreover, by definition, every element $z_h^k \in \LL_h^{N+1}$ is
$\epsilon$-dependent on $N = L_h/\textnormal{dim}_E(\FF, \epsilon) - 1$ disjoint
subsequences in $\ZZ_h^k \cap \LL^h$.

Furthermore, since $\Bar{b}_h^{k}(s_h^k, a_h^k) > \epsilon$ for all $z_h^k \in \LL_h$
there must exist $f_1, f_2 \in \FF$ such that $|f_1(s_h^k, a_h^k) - f_2(s_h^k, a_h^k)| > \epsilon$ and
$||f_1 - f_2||^2_{\ZZ_h^k} \leq 100\beta$. Hence, since $z_h^k \in \LL_h^{N+1}$ is
$\epsilon$-dependent on $N$ disjoint subsequences $\LL_h^j$ and for each such
subsequence, by the definition of $\epsilon$-dependence, it holds that $||f_1 - f_2||_{\LL_h^j} > \epsilon^2$ we have that

\begin{align*}
    N \epsilon^2 \leq ||f_1 - f_2||^2_{\ZZ_h^k} \leq 100\beta \implies
    (L_h/\textnormal{dim}_E(\FF,\epsilon)-1)\epsilon^2 \leq 100\beta \implies L_h \leq \left( \frac{100\beta}{\epsilon^2} + 1\right)\textnormal{dim}_E(\FF, \epsilon).
\end{align*}

We now pick a permutation $\Bar{b}_1 \geq \Bar{b}_2 \geq \ldots \geq \Bar{b}_{|K'|}$ of the bonus
functions $\{\Bar{b}_h^k(s^k_h, a^k_h) \}_{k \in K'}$. For all $\Bar{b}_k \geq 1/|K'|$ it holds
that

\begin{align*}
    k \leq \left(\frac{100\beta}{\Bar{b}_k^2} + 1\right)\textnormal{dim}_E(\FF, \Bar{b}_k) &\leq \left(\frac{100\beta}{\Bar{b}_k^2} + 1\right)\textnormal{dim}_E(\FF, 1/K') \implies \\
    \Bar{b}_k &\leq \left(\frac{k}{\textnormal{dim}_E(\FF, 1/K')} -1  \right)^{-1/2}\sqrt{100\beta}. 
\end{align*}

Moreover, notice that we get by definition that $\Bar{b}_k \leq H + 1$. Hence, we have that

\begin{align*}
    \sum_{i=1}^{|K'|} \Bar{b}_h^{k_i}(s_h^{k_i}, a_h^{k_i}) &= \sum_{i: \Bar{b}_{k_i} < 1/|K'|} \Bar{b}_h^{k_i}(s_h^{k_i}, a_h^{k_i}) + \sum_{i: \Bar{b}_{k_i} \geq 1/|K'|} \Bar{b}_h^{k_i}(s_h^{k_i}, a_h^{k_i}) \\
    &\leq   |K'| \cdot 1/|K'| + \sum_{i: \Bar{b}_{k_i} \geq 1/|K'|, i \leq  \textnormal{dim}_E(\FF, 1/|K'|)} \Bar{b}_h^{k_i}(s_h^{k_i}, a_h^{k_i}) + \sum_{i: \Bar{b}_{k_i} \geq 1/|K'|, |K'| \geq i >  \textnormal{dim}_E(\FF, 1/|K'|)} \Bar{b}_h^{k_i}(s_h^{k_i}, a_h^{k_i}) \\
    &\leq 1 + (H+1) \cdot \dim_E(\FF, 1/|K'|) + \sum_{|K'| \geq i >  \textnormal{dim}_E(\FF, 1/|K'|)} \left(\frac{k}{\textnormal{dim}_E(\FF, 1/|K'|)} -1  \right)^{-1/2}\sqrt{100\beta} \\
    &\leq1 + (H+1) \cdot \dim_E(\FF, 1/|K'|) + C \sqrt{\textnormal{dim}_E(\FF, 1/|K'|) |K'|\beta} \\
    &\leq1 + (H+1) \cdot \dim_E(\FF, 1/T) + C \sqrt{\textnormal{dim}_E(\FF, 1/T) |K'|\beta}
\end{align*}

for some constant $C > 0$, where the second to last inequality can be obtained by bounding the summation by the integral and the last one by the definition of the eluder dimension.  Summing up all the inequalities for $h \in [H]$, we get the result.
\end{proof}

We need the Azuma-Hoeffding inequality to bound a martingale difference sequence.
For completeness, we present it here as well.

\begin{lemma}\citep{cesa2006prediction}\label{lem:azuma-hoeffding inequality}
Let $\{ x_i\}_{i=1}^n$ be a martingale difference sequence with respect to some filtration $\{ \FF_i\}$ for which $|x_i| \leq M$ for some constant $M$, $x_i$ is $\FF_{i+1}$ measurable and $\E[x_i | \FF_i] = 0$. Then, for any $0 < \delta < 1$, we have that with probability at least $1-\delta$ it holds that 
\begin{align*}
    \sum_{i=1}^n x_i \leq M \sqrt{2n\log(1/\delta)}.
\end{align*}
\end{lemma}

The following lemma that appears in~\cite{kong2021online}
shows that the estimate of the $Q$-function upper bounds the optimal one.

\begin{lemma}\citep{kong2021online}\label{lem:single step bellman error}
With probability at least $1-\delta/2$ we have that for all $(k, h) \in [K] \times [H]$ and all $(s,a) \in \S
\times \AA$
\begin{align*}
    Q^*_h(s,a) \leq Q^k_h(s,a) \leq \Bar{f}_h^k(s,a) + 2 b_h^k(s,a)
\end{align*}
where $\Bar{f}_h^k(\cdot, \cdot) = \sum_{s' \in \S}P_h(s'|\cdot, \cdot) V_{h+1}^k(s') + r_h(\cdot, \cdot)$.

\end{lemma}

We now present a lemma that bounds the
number of rounds that the suboptimilaty gap falls in some interval.
It is inspired by Lemma~6.2~\citep{he2021logarithmic}. 

\begin{lemma}\label{lem:indicator bound on every interval}
If we pick $$\beta = CH^2 \log(T\NN(\FF, \delta/T^2)/\delta)\dim_E(\FF, 1/T) \log^2T \log \left( \CC(\S \times \AA, \delta/ T^2)T/\delta)\right),$$
then there exists a constant $\tilde{C}$ such that for all $h \in [H], n \in [N]$ with probability at least $1-2K\delta$, we have that

\begin{align*}
    \sum_{k=1}^K \mathbbm{1}[V^*_h(s_h^k) - Q_h^{\pi_k}(s_h^k, a_h^k) \geq 2^n \gap_{\min}] \leq \frac{\tilde{C} H^4 \dim^2_E(\FF, 1/T) \log^2T \log(T\NN(\FF,\delta/T^2)/\delta)\log(\CC(\S \times \AA, \delta/T^2)T/\delta)}{4^n \gap_{\min}^2}.
\end{align*}
\end{lemma}

 \begin{proof}
 We keep $h$ fixed. 
 

We denote by $K'$ the set of episodes where the gap at step $h$ is at least $2^n$, i.e.
\begin{align*}
    K' = \left\{k \in [K]: V^*_h(s_h^k) - Q_h^{\pi_k}(s_h^{k}, a_h^{k}) \geq 2^n \gap_{\min}\right\}.
\end{align*}

The goal is to bound the quantity $\sum_{i=1}^{|K'|}\left(Q_h^{k_i}(s_h^{k_i},a_h^{k_i}) - Q_h^{\pi_{k_i}}(s_h^{k_i}, a_h^{k_i})\right)$ from below and above with functions
$f_1(|K'|), f_2(|K'|)$ and then use the fact that $f_1(|K'|) \leq f_2(|K'|)$ to derive an upper bound on $|K'|$.

For the lower bound, we have that
\begin{align*}
    \sum_{i=1}^{|K'|}\left(Q_h^{k_i}(s_h^{k_i},a_h^{k_i}) - Q_h^{\pi_{k_i}}(s_h^{k_i}, a_h^{k_i})\right) &\geq \sum_{i=1}^{|K'|}\left(Q_h^{k_i}(s_h^{k_i},\pi_h^*(s_h^{k_i})) - Q_h^{\pi_{k_i}}(s_h^{k_i}, a_h^{k_i})\right) 
    \\& \geq \sum_{i=1}^{|K'|}\left(Q_h^*(s_h^{k_i},\pi_h^*(s_h^{k_i})) - Q_h^{\pi_{k_i}}(s_h^{k_i}, a_h^{k_i})\right) 
    \\&= \sum_{i=1}^{|K'|}\left(V_h^*(s_h^{k_i}) - Q_h^{\pi_{k_i}}(s_h^{k_i}, a_h^{k_i})\right) \geq 2^n \gap_{\min} |K'|,
\end{align*}

where the first inequality holds by the definition of the policy $\pi_{k_i}$, the second one follows because $Q^{k_i}_h(\cdot, \cdot)$ is an optimistic estimate of $Q^*_h(\cdot, \cdot)$ which happens with probability at least $1-\delta/2$ (see Lemma~\ref{lem:single step bellman error}) and the third one by the definition of $k_i$.

We get the upper bound on this quantity in the following way. For any $h' \in [H]$ we have
\begin{align*}
    Q_{h'}^k(s_{h'}^k, a_{h'}^k) - Q_{h'}^{\pi_k}(s_{h'}^k, a_{h'}^k) &\leq \sum_{s'\in \S}P_{h'}(s'|s_{h'}^k,a_{h'}^k)V^k_{h'+1}(s') + r_{h'}(s_{h'}^k, a_{h'}^k) + 2b_{h'}^k(s_{h'}^k, a_{h'}^k) - Q_{h'}^{\pi_k}(s_{h'}^k, a_{h'}^k)  \\
    &=\left\langle P_{h'}(\cdot|s_{h'}^k, a_{h'}^k), V_{h'+1}^k - V_{h'+1}^{\pi_k}\right\rangle + 2b_{h'}^k(s_{h'}^k, a_{h'}^k) \\
    &= V_{h'+1}^k(s_{h'+1}^k) - V_{h'+1}^{\pi_k}(s_{h'+1}^k) + \epsilon_{h'}^k + 2b_{h'}^k(s_{h'}^k, a_{h'}^k)  \\
    &=Q_{h'+1}^k(s_{h'+1}^k, a_{h'+1}^k) - Q_{h'+1}^{\pi_k}(s_{h'+1}^k, a_{h'+1}^k) + \epsilon_{h'}^k + 2b_{h'}^k(s_{h'}^k, a_{h'}^k)
\end{align*}
where we define $\epsilon_{h'}^k = \left\langle P_{h'}(\cdot|s_{h'}^k, a_{h'}^k), V_{h'+1}^k - V_{h'+1}^{\pi_k}\right\rangle - (V_{h'+1}^k(s_{h'+1}^k) - V_{h'+1}^{\pi_k}(s_{h'+1}^k))$ and the inequality follows from Lemma~\ref{lem:single step bellman error}.

We now take the summation over all $k \in |K'|, h \leq h' \leq H$ and we get

\begin{align*}
    \sum_{i=1}^{|K'|} \left( Q_h^{k_i}(s_h^{k_i}, a_h^{k_i}) - Q_h^{\pi_{k_i}}(s_h^{k_i}, a_h^{k_i}) \right) \leq \sum_{i=1}^{|K'|} \sum_{h'=h}^H\epsilon_{h'}^{k_i} + \sum_{i=1}^{|K'|} \sum_{h'=h}^H b_{h'}^{k_i}(s_{h'}^{k_i}, a_{h'}^{k_i}).
\end{align*}

We will bound each of the two terms on the RHS separately. 

For the first term, we notice that $x_j = \left\langle P_{j}(\cdot|s_{j}^{k_i}, a_{j}^{k_i}), V_{j+1}^{k_i} - V_{j+1}^{\pi_{k_i}}\right\rangle - (V_{j+1}^{k_i}(s_{j+1}^{k_i}) - V_{j+1}^{\pi_{k_i}}(s_{j+1}^{k_i}))$ forms a martingale difference sequence with zero mean and $|x_j| \leq 2H$. Hence, we can use Lemma~\ref{lem:azuma-hoeffding inequality} and  that for each $k\in K'$, with probability at least $1-\delta$ we have that

\begin{align*}
    \sum_{i=1}^k \sum_{j=1}^H \left( \left\langle P_{j}(\cdot|s_{j}^{k_i}, a_{j}^{k_i}), V_{j+1}^{k_i} - V_{j+1}^{\pi_{k_i}}\right\rangle - \left(V_{j+1}^{k_i}(s_{j+1}^{k_i}) - V_{j+1}^{\pi_{k_i}}(s_{j+1}^{k_i})\right) \right) \leq \sqrt{8k H^3 \log(1/\delta)}.
\end{align*}

If we take the union bound over all $k\in [K]$ we have that with probability at least $1-|K'|\delta$

\begin{align*}
    \sum_{i=1}^{|K'|} \sum_{h'=h}^H\epsilon_{h'}^{k_i} \leq \sqrt{8 |K'| H^3 \log(1/\delta)}.
\end{align*}

We now focus on the second term. Using Lemma~\ref{lem:bound of bonus functions over all rounds} we get that

\begin{align*}
    \sum_{i=1}^{|K'|} \sum_{h'=h}^H b_{h'}^{k_i}(s_{h'}^{k_i}, a_{h'}^{k_i}) \leq H + H(H+1)\dim_E(\FF, 1/T) + C H\sqrt{\textnormal{dim}_E(\FF, 1/T) |K'|\beta}
\end{align*}

and this happens with probability at least $1-\delta/32$. Hence, combining the upper and lower bound of $$\sum_{i=1}^{|K'|}\left(Q_h^{k_i}(s_h^{k_i},a_h^{k_i}) - Q_h^{\pi_{k_i}}(s_h^{k_i}, a_h^{k_i})\right)$$ we get that

\begin{align*}
    2^n \gap_{\min}|K'| \leq \sqrt{8 |K'| H^3 \log(1/\delta)} +  H + H(H+1)\dim_E(\FF, 1/T) + C H\sqrt{\textnormal{dim}_E(\FF, 1/T) |K'|\beta}.
\end{align*}

Solving for $|K'|$ gives us that

\begin{align*}
    |K'| \leq \frac{\tilde{C} H^4 \dim^2_E(\FF, 1/T) \log^2T \log(T\NN(\FF,\delta/T^2)/\delta)\log(\CC(\S \times \AA, \delta/T^2)T/\delta)}{4^n \gap_{\min}^2}.
\end{align*}

 \end{proof}

We are now ready to prove Lemma~\ref{lem:sum of suboptimalities over all K for some h}.

\begin{prevproof}{Lemma}{lem:sum of suboptimalities over all K for some h}
Throughout this proof we condition on the event described in Lemma~\ref{lem:indicator bound on every interval} which happens with probability at least $1-2K\delta$. Since $\gap_{\min} > 0$ whenever we do not take the optimal action, we have that either $V_h^*(s_k) - Q_h^*(s_h^k, a_h^k) = 0$ or $V_h^*(s_k) - Q_h^*(s_h^k, a_h^k) \geq \gap_{\min}$. Our approach is to divide the interval $[0, H]$ into $N = \lceil \log(H/\gap_{\min})\rceil$ intervals and count the number of $V_h^*(s_k) - Q_h^*(s_h^k, a_h^k)$ that fall into each interval. Notice that for every $V_h^*(s_k) - Q_h^*(s_h^k, a_h^k)$ that falls into interval $i$ we can get an upper bound of $V_h^*(s_k) - Q_h^*(s_h^k, a_h^k) \leq 2^i\gap_{\min}$ and this upper bound is essentially tight. Hence, we have that

\begin{align*}
    \sum_{k=1}^K \left( V_h^*(s_h^k) - Q_h^*(s_h^k, a_h^k) \right) &\leq \sum_{i=1}^N \sum_{k=1}^K \mathbbm{1}\left[ 2^i\gap_{\min} \geq V_h^*(s_h^k) - Q_h^*(s_h^k, a_h^k) \geq 2^{i-1}\gap_{\min} \right] \cdot 2^i \gap_{\min} \\
    &\leq \sum_{i=1}^N \sum_{k=1}^K \mathbbm{1}\left[  V_h^*(s_h^k) - Q_h^*(s_h^k, a_h^k) \geq 2^{i-1}\gap_{\min} \right] \cdot 2^i \gap_{\min} \\
    &\leq \sum_{i=1}^N \frac{\tilde{C} H^4 \dim^2_E(\FF, 1/T) \log^2T \log(T\NN(\FF,\delta/T^2)/\delta)\log(\CC(\S \times \AA, \delta/T^2)T/\delta)}{4^{i-1} \gap_{\min}^2} \cdot 2^i \gap_{\min} \\
    &= \sum_{i=1}^N \frac{\tilde{C} H^4 \dim^2_E(\FF, 1/T) \log^2T \log(T\NN(\FF,\delta/T^2)/\delta)\log(\CC(\S \times \AA, \delta/T^2)T/\delta)}{2^{i} \gap_{\min}} \\
    &\leq
    \frac{C H^4 \dim^2_E(\FF, 1/T) \log^2T \log(T\NN(\FF,\delta/T^2)/\delta)\log(\CC(\S \times \AA, \delta/T^2)T/\delta)}{ \gap_{\min}}
\end{align*}

where the first inequality holds by the definition of the intervals, the second due to the properties of the indicator function, the third because of Lemma~\ref{lem:indicator bound on every interval} and in the last two steps we just manipulate the constants. 
\end{prevproof}
\section{Proof of Theorem~\ref{thm:model-based main result}}
\label{sec:model-based proof}
In this section, our main goal is to prove 
Theorem~\ref{thm:model-based main result}.
We work with Assumption~\ref{as:model-based-assumption}. We follow
the same regret decomposition as in Appendix~\ref{sec:linear-mdp-proof}.

We first present a lemma that is crucial in bounding the regret 
of the algorithm.

\begin{lemma}\label{lem:linear mixture sum of suboptimalities over all K for some h}
If we pick $$\beta = 4H^2\log(2\NN(\FF,1/T)/\delta) + 4/H\left(C + \sqrt{H^2/4\log(4(K(K+1)/\delta)})\right),$$
for come constant C and for $h \in [H]$, then we have that with probability at least $1-(K+3)\delta$
\begin{align*}
    \sum_{k=1}^K \left( V_h^*(s_h^k) - Q_h^*(s_h^k, a_h^k) \right) \leq\frac{C H^4 \log(T\NN(\FF,1/T)/\delta) \dim^2_E(\FF, 1/T) }{ \gap_{\min}}.
\end{align*}

\end{lemma}

We are now ready to state the regret bound of our algorithm.
\begin{lemma}\label{lem:model-based regret guarantee}
There exists a constant $C$ and proper values of the parameter $\beta$ of Algorithm~\ref{alg:main algorithm} such that with probability at least $1- \lceil \log T \rceil e^{-\tau} - H(K+3)\delta$ the regret of the algorithm is bounded by
\begin{align*}
    \textnormal{Regret}(K) \leq \frac{C H^5 \log(T\NN(\FF,1/T)/\delta) \dim^2_E(\FF, 1/T) }{ \gap_{\min}}  + \frac{16 H^2 \tau}{3} + 2.
\end{align*}

The value of the parameter is $$\beta = 4H^2\log(2\NN(\FF,1/T)/\delta) + 4/H\left(C + \sqrt{H^2/4\log(4(K(K+1)/\delta)})\right).$$

where $\NN(\FF,1/T) = \arg\max_{h \in [H]} \NN(\FF_h,1/T)$.

In particular, the dependence of the regret in the time horizon $T$ is logarithmic.
\end{lemma}

\begin{proof}
Throughout the proof, we condition on the events described in Lemma~\ref{lem:model-free-regret-decomposition},~\ref{lem:linear mixture sum of suboptimalities over all K for some h} which happen with probability at least $1- \lceil \log T \rceil e^{-\tau} - H(K+3)\delta$.

From Lemma~\ref{lem:model-free-regret-decomposition} we have that
\begin{align*}
    \text{Regret}(K) \leq 2 \sum_{k=1}^K\sum_{h=1}^H \gap_h(s_h^k, a_h^k) + \frac{16 H^2 \tau}{3} + 2.
\end{align*}

We can bound the first term on the RHS using Lemma~\ref{lem:linear mixture sum of suboptimalities over all K for some h} as follows
\begin{align*}
    2 \sum_{k=1}^K\sum_{h=1}^H \gap_h(s_h^k, a_h^k) = 2 \sum_{k=1}^K\sum_{h=1}^H  \left( V_h^*(s_h^k) - Q_h^*(s_h^k, a_h^k) \right) \leq \frac{2C H^5 \log(T\NN(\FF,1/T)/\delta) \dim^2_E(\FF, 1/T) }{ \gap_{\min}}.
\end{align*}

This gives us the result.

\end{proof}

The only thing that we need to do now is to bound the number of rounds that
we update our policy. Since we are using exactly the same sensitivity score and
update probability as in \cite{kong2021online}, this follows from their result.

\begin{lemma}
\label{lem:model-based-bound on summation of sensitivity}
\citep{kong2021online}
With probability at least $1-\delta/32$ for any fixed $h \in [H]$ we have that the sub-sampled dataset $\hat{Z}_h^k, k \in [K]$ changes at most $$S_{\max} = C\cdot\log(T\NN(F_h,\sqrt{\delta/(64T^3)})/\delta)\dim_E(\FF_h,1/T)\log^2T$$ times.
\end{lemma}

We are now ready to prove Theorem~\ref{thm:model-based main result}.

\begin{prevproof}{Theorem}{thm:model-based main result}

=The proof of this theorem follows by combining Lemma~\ref{lem:model-based regret guarantee}
and taking a union on the result of Lemma~\ref{lem:model-based-bound on summation of sensitivity}
and setting the error probability accordingly.
\end{prevproof}

\subsection{Supporting Lemmas: Theorem~\ref{thm:model-based main result}}
In this section our goal is to prove the supporting lemmas of Theorem~\ref{thm:model-based main result}.

Recall that our approach 
is to modify the algorithm in~\cite{kong2021online} to work in this 
setting and use a similar analysis as in 
Appendix~\ref{sec:linear-mdp-proof}. Unlike 
Appendix~\ref{sec:linear-mdp-proof}
where we approximate the optimal Q-function, here we try to 
estimate
the true transition kernel. 
Let 
$\ZZ_h^k = \{\left(s_h^{\tau}, a_h^{\tau}, V_{h+1}^{\tau}(\cdot)\right)\}_{\tau \in [k-1]}$ be the dataset up to episode $k$ and $\hat{\ZZ}_h^k$ the sub-sampled dataset.
In each episode $k$, we update our policy whenever we add an element
in the dataset for some $h \in [H]$. Recall that whenever we perform an update our policy becomes:
\begin{align*}
    &Q_{H+1}^k(s, a) = 0, \\
    &V_{H+1}^k(s) = 0, \\
    &Q_h^k(s,a) = \min\{r_h(s, a) +  \langle\hat{P}^k_h(\cdot| s, a), V^k_{h+1}\rangle + b_h^k(s,a), H\}, \\
    &V_h^k(s) = \max_{a \in \AA} Q_h^k(s, a)
\end{align*}
for some $\hat{P}_h^k, b_h^k(\cdot, \cdot)$ that we will define shortly. 
We get the policy $\pi^k_h(s)$ by picking greedily the action that maximizes
the estimate $Q_h^k(s,a)$.

The least-squares estimate of the model is
\begin{align*}
    \hat{P}_h^{k+1} = \arg\min_{P \in \PP_h} \sum_{k'=1}^k \left(\langle P(\cdot|s_h^{k'}, a_h^{k'}), V_{h+1}^{k'}\rangle - y_h^{k'} \right)^2, y_h^{k'} = V_{h+1}^{k'}(s_{h+1}^{k'}).
\end{align*}






Recall the definition of the function class 
that we use in the derivation of our results.
\begin{definition}
\label{def:model-based function class}
Let $\VV$ be the set all measurable functions that are bounded by $H$. We now let $f : \S \times \AA \times \VV \rightarrow \R$ and define the following set:
\begin{align}
    \FF_h = \left\{f: \exists P_h \in \PP_h \text{ so that } f(s, a, V) = \int_\S P_h(s'| s,a) V(s') ds', \forall (s, a, V) \in \S \times \AA \times \VV \right\}.
\end{align}
\end{definition}

Recall that in this setting the norm of a function
with respect to a dataset $\ZZ$ is
$$||f||_{\ZZ} = \sqrt{\sum_{z = (s_z,a_z,V_z(\cdot)) \in \ZZ} \left(f(s_z, a_z, V_z(\cdot)\right)^2}.$$

Recall also that the bonus function is
\begin{align*}
    b_h^k(s,a) = \sup_{f_1, f_2 \in \FF_h: \min\{||f_1 - f_2||_{\hat{\ZZ}_h^k}, T(H+1)^2\} \leq \beta} |f_1(s,a,V_{h+1}^k(\cdot)) - f_2(s,a,V_{h+1}^k(\cdot))|.
\end{align*}

The parameter $\beta$ will be defined later in a way that 
will ensure optimism. 





First, we need to show that at for every $k \in [K], h \in [H]$, the
sub-sampled dataset approximates the original one. Our approach
is inspired by \citep{kong2021online}.

 We
define the following quantities

\begin{align*}
    \underline{\CC}_h^k(\alpha) &= \left\{(f_1, f_2)\in \FF_h \times \FF_h: ||f_1 - f_2||^2_{\ZZ_h^k} \leq \alpha/100\right\}\\
    \hat{\CC}_h^k(\alpha) &=
    \left\{(f_1, f_2)\in \FF_h \times \FF_h: \min\{||f_1 - f_2||^2_{\hat{\ZZ}_h^k},T(H+1)^2\} \leq \alpha\right\}\\
    \overline{\CC}_h^k(\alpha) &=\left\{(f_1, f_2)\in \FF_h \times \FF_h: ||f_1 - f_2||^2_{\ZZ_h^k} \leq 100\alpha\right\}.
\end{align*}

We also let 
\begin{align*}
    \underline{b}_h^k(s,a) = \sup_{f_1, f_2 \in \underline{\CC}_h^k(\beta)} |f_1(s,a,V_{h+1}^k(\cdot)) - f_2(s,a,V_{h+1}^k(\cdot))|\\
     \overline{b}_h^k(s,a) = \sup_{f_1, f_2 \in \overline{\CC}_h^k(\beta)} |f_1(s,a,V_{h+1}^k(\cdot)) - f_2(s,a,V_{h+1}^k(\cdot))|.
\end{align*}

Our goal is to show that $\underline{\CC}_h^k(\alpha) \subseteq \hat{\CC}_h^k(\alpha) \subseteq \overline{\CC}_h^k(\alpha)$ with high probability. Let $\EE_h^k(\alpha)$ denote the event that this holds. We also denote by $\EE_h^k = \cap_{n=0}^\infty \EE_h^k(100^n\beta)$. This event will show us that $\hat{\ZZ}_h^k$ is a good approximation to $\ZZ_h^k$.\\
Notice that whenever this happens, it holds that
$\underline{b}_h^k(s,a) \leq b_h^k(s,a) \leq \overline{b}_h^k(s,a)$.

The following lemma which is inspired by~\cite{kong2021online} establishes that fact.

\begin{lemma}\label{lem: mixture mdp: sub-sampled dataset is good approximation to original one}
The probability that all the events $\EE_h^k$ happen satisfies
\begin{align*}
    \Pr\left(\bigcap_{k=1}^K \bigcap_{h=1}^H \EE_h^k \right) \geq 1 - \delta.
\end{align*}
\end{lemma}

To prove Lemma~\ref{lem: mixture mdp: sub-sampled dataset is good approximation to original one} we need the following concentration inequality proved in~\cite{freedman1975tail}.

\begin{lemma}\label{lem:freedman's concetration lemma}
Let $\{Y_i\}_{i \in \N}$ be a real-valued martingale with difference sequence $\{X_i\}_{i \in \N}$. Let $R$ be a uniform bound on $X_i$. Fix some $n \in \N$ and let $\sigma^2$ be a number such that
\begin{align*}
    \sum_{i=1}^n \E[X_i^2|Y_0, \ldots, Y_{i-1}] \leq \sigma^2.
\end{align*}

Then, for all $t \geq 0$ we have that
\begin{align*}
    \Pr(|Y_n - Y_0| \geq t) \leq 2\exp{\left\{ -\frac{t^2/2}{\sigma^2 + Rt/3}\right\}}.
\end{align*}
\end{lemma}

Moreover, we need a bound on the number of elements that are in the sub-sampled dataset.
This is established in~\cite{kong2021online}.

\begin{lemma}\citep{kong2021online}
\label{lem:mixture-mdp-bound-sampled-dataset}
We have that with probability at least $1 - \delta/(64T)$, we have 
    $|\hat{\ZZ}_h^k| \leq 64T^3/\delta$
for all $\delta > 0$.

\end{lemma}

The subsequent lemma shows that, indeed, whenever $\EE_h^k$ happens
the sub-sampled dataset is a good approximation of the original one. It
was proved in \citep{kong2021online}.

\begin{lemma}\citep{kong2021online}
\label{lem:linear-mixture-mdp-approximation-under-event}
Whenever the event $\EE_h^k$ happens, it holds that
\begin{align*}
    \frac{1}{10000} ||f_1 - f_2||^2_{\ZZ_h^k} \leq \min\{||f_1-f_2||^2_{\hat{\ZZ}_h^k}, T(H+1)^2 \} \leq 10000 ||f_1 - f_2||^2_{\ZZ_h^k}, \textnormal{ if } ||f_1 - f_2||^2_{\ZZ_h^k} > 100\beta
\end{align*}
and 
\begin{align*}
    \min\{||f_1-f_2||^2_{\hat{\ZZ}_h^k}, T(H+1)^2 \} \leq 10000\beta, \textnormal{ if } ||f_1 - f_2||^2_{\ZZ_h^k} \leq 100\beta.
\end{align*}
\end{lemma}

To establish our result, we need the following  lemma. The proof follows the approach of~\cite{kong2021online}. We present it here for completeness.

\begin{lemma}\label{lem:helper lemma to prove approximation of sub-sampled dataset}

For any $\alpha \in [\beta, T(H+1)^2]$, a fixed $h \in [H]$ and $k \in [K]$ we have the following bound for the probability that all the events $\{\EE_h^i\}_{i\leq k-1}$ happen and the last one does not happen
\begin{align*}
    \Pr\left( \left(\bigcap_{i=1}^{k-1} \EE_h^i\right) \EE_h^k(\alpha)^c\right) \leq  \delta/(32 T^2).
\end{align*}

\end{lemma}

\begin{proof}
Let $C_1$ be the quantity the sensitivity in the sampling probability. We fix some $h \in [H]$ throughout the proof.

We consider a fixed pair of functions $f_h^1, f_h^2$ in the discretized set $\CC(\FF_h,\sqrt{\delta/(64T^3)})$ and for $i \geq 2$ we let $$Z_i = \max\left\{||f_h^1 - f_h^2||^2_{\ZZ_h^i}, \min\{||f_h^1 - f_h^2||^2_{\hat{\ZZ}_h^{i-1}}, T(H+1)^2\}\right\}.$$

We also define
\begin{align*}
    Y_i = 
    \begin{cases}
    \frac{1}{p_{z_h^{i-1}}}(f_h^1(z_h^{i-1}) - f_h^2(z_h^{i-1}))^2 & \text{$z_h^{i-1}$ is added to $\hat{\ZZ}_h^i$ and $Z_i \leq 2000000 \alpha$}\\
    0, & \text{$z_h^{i-1}$ is not added to $\hat{\ZZ}_h^i$ and $Z_i \leq 2000000 \alpha$} \\
    (f_h^1(z_h^{i-1}) - f_h^2(z_h^{i-1}))^2 & \text{otherwise}
    \end{cases}
\end{align*}

Let $\mathbb{F}_i$ be the filtration that $Y_i$ is adapted to. Our goal is to use Freedman's inequality (i.e. Lemma~\ref{lem:freedman's concetration lemma}) for $Y_i$.
Notice that $\E[Y_i | \mathbb{F}_i] = (f_h^1(z_h^{i-1}) - f_h^2(z_h^{i-1}))^2$. Now we focus on the variance of $Y_i$. Notice that if $p_{z_h^{i-1}} = 1$ or $Z_i > 2000000\alpha$ then $Y_i$ is deterministic so $Y_i - \E[Y_i | \mathbb{F}_{i-1}] = \text{Var}[Y_i - \E[Y_i | \mathbb{F}_{i-1}]] = 0$. For the other case, recall that $$p_{z_h^i} = \min\{1, C'\cdot \text{sensitivity}_{\hat{\ZZ}_h^{i-1}, \FF_h}(z_h^i) \cdot \log(T\NN(\FF,\sqrt{\delta/64T^3})/\delta) \} = \min\{1, C_1\cdot \text{sensitivity}_{\hat{\ZZ}_h^{i-1}, \FF_h}(z_h^i) \}$$ and 
\begin{align*}
    \text{sensitivity}_{\ZZ, \FF}(z) = \min\left\{\sup_{f_1, f_2 \in \FF}  \frac{(f_1(z) - f_2(z))^2}{\min\{||f_1 - f_2||_{\mathcal{\ZZ}}, T(H+1)^2\}+\beta}, 1\right\}
\end{align*}

Since $p_{z_h^{i-1}} < 1 \implies C_1\cdot \text{sensitivity}_{\hat{\ZZ}^{i-1}_h, \FF_h}(z_h^i) < 1$. We consider two cases. If $Y_i \neq 0$ we can see that $Y_i \geq \E[Y_i|\mathbb{F}_i]$ so $|Y_i - \E[Y_i] | \leq Y_i$. Moreover,
\begin{align*}
    Y_i &\leq  \frac{(f_h^1(z_h^{i-1}) - f_h^2(z_h^{i-1}))^2}{p_{z_h^{i-1}}} \\
    &\leq \frac{(f_h^1(z_h^{i-1}) - f_h^2(z_h^{i-1}))^2}{C_1\sup_{f_1, f_2 \in \FF_h}  \frac{(f_1(z_h^{i-1}) - f_2(z_h^{i-1}))^2}{\min\{||f_1 - f_2||_{\mathcal{\hat{\ZZ}}_h^{i-1}}, T(H+1)^2\}+\beta}} \\
    &\leq \frac{(f_h^1(z_h^{i-1}) - f_h^2(z_h^{i-1}))^2 \min\{||f_1 - f_2||_{\mathcal{\hat{\ZZ}}_h^{i-1}}, T(H+1)^2\}+\beta\}}{C_1(f_h^1(z_h^{i-1}) - f_h^2(z_h^{i-1}))^2} \\
    &= \left(\min\{||f_1 - f_2||_{\mathcal{\hat{\ZZ}}_h^{i-1}}, T(H+1)^2\}+\beta\right) \cdot 1/C_1 \\
    &\leq
    2000001\alpha/C_1 < 3000000\alpha/C_1
\end{align*}

Thus, we see that $|Y_i - \E[Y_i] | \leq 3000000\alpha/C_1$. On the other hand, we can see that if $Y_i = 0$ then
$|\E_{i-1}[Y_i] - Y_i| = (f_h^1(z_h^{i-1}) - f_h^2(z_h^{i-1}))^2$ and the inequality
we derived above still holds.

For the variance, we can see that
\begin{align*}
    \text{Var}[Y_i - \E[Y_i | \mathbb{F}_i]|\mathbb{F}_i] &= p_{z_h}^{i-1} \left( \frac{1}{p_{z_h^{i-1}}}(f_h^1(z_h^{i-1}) - f_h^2(z_h^{i-1}))^2\right)^2 + (1-p_{z_h}^{i-1})\cdot 0 \\
     &\leq \frac{1}{p_{z_h^{i-1}}}(f_h^1(z_h^{i-1}) - f_h^2(z_h^{i-1}))^4 \\
     &\leq 3000000\alpha \left(f_h^1(z_h^{i-1}) - f_h^2(z_h^{i-1})\right)^2 /C_1
\end{align*}

where the first equality follows from the definition, the first inequality is trivial 
and the third one from the inequality we derived above. Let $k'$ be the maximum number $\leq k$ such that $Z_{k'} \leq 2000000\alpha$. Summing up the above inequalities for $i=2,\ldots,k$ we get
\begin{align*}
    \sum_{i=2}^k \text{Var}[Y_i - \E[Y_i | \mathbb{F}_i]|\mathbb{F}_i] &= \sum_{i=2}^{k'} \text{Var}[Y_i - \E[Y_i | \mathbb{F}_i]|\mathbb{F}_i] \\
    &\leq \frac{3000000\alpha}{C_1} \sum_{i=2}^{k'}(f_h^1(z_h^{i-1}) - f_h^2(z_h^{i-1}))^2 \\
    &\leq \frac{3000000\alpha \cdot 2000000\alpha}{C_1} \\
    &\leq
    \frac{(3000000\alpha)^2}{C_1}
\end{align*}

where the the first equality follows from the fact that for $i > k'$ the random variable is deterministic, the first inequality follows by the summation of the previous one and the second one by the fact that $\sum_{i=2}^{k'}(f_h^1(z_h^{i-1}) - f_h^2(z_h^{i-1}))^2 \leq ||f_h^1 - f_h^2||_{\ZZ_h^{k'}} \leq Z_{k'}$.

We are now ready to use Freedman's inequality (Lemma~\ref{lem:freedman's concetration lemma}) with $R = \frac{3000000\alpha}{C_1}, \sigma^2 = \frac{(3000000\alpha)^2}{C_1}$. We get
\begin{align*}
    \Pr \left(\left|\sum_{i=1}^{k}(Y_i - \E[Y_i | \mathbb{F}_i])\right| \geq \alpha/100 \right) &= \Pr \left(\left|\sum_{i=1}^{k'}(Y_i - \E[Y_i | \mathbb{F}_i])\right| \geq \alpha/100 \right) \\
    &\leq
    2\exp\left\{-\frac{(\alpha/100)^2/2}{(3000000\alpha)^2/C_1 + \alpha^2 3000000/300C_1}\right\} \\
    &= 2\exp\left\{-\frac{C_1}{20000(3000000+10000)}\right\} \\
    &=
    2\exp\left\{-\frac{C\log( T \NN(\FF_h,\sqrt{\delta/64T^3})/\delta)}{20000(3000000+10000)}\right\} \\
    &= 2\exp\left\{-\frac{C(\log((T \NN(\FF_h,\sqrt{\delta/64T^3})/\delta)^2))}{40000(3000000+10000)}\right\} \\
    &\leq (\delta/64T^2)/ (\NN(\FF_h, \sqrt{\delta/64T^3}))^2
\end{align*}

for some choice of $C$. Now we can take a union bound over all the functions in the discretized set and conclude that with probability at least $1-\delta/(64T^2)$ we have that 
\begin{align*}
    \left|\sum_{i=1}^{k}(Y_i - \E[Y_i | \mathbb{F}_i])\right| \leq \alpha/100
\end{align*}

for all pairs of functions in this set.
We condition on this event and on the event in Lemma~\ref{lem:mixture-mdp-bound-sampled-dataset}. We first show that when this event happens,  we have that $\underline{\CC}_h^k(\alpha) \subseteq \hat{\CC}_h^k(\alpha)$. Consider $f_1, f_2 \in \underline{\CC}_h^k(\alpha)$. We know that there exist $f_1',f_2' \in \CC(\FF,\sqrt{\delta/(64T^3)}) \times \CC(\FF,\sqrt{\delta/(64T^3)})$ with $||f_1 - f_1'||_{\infty}, ||f_2 - f_2'||_{\infty} \leq \sqrt{\delta/64T^3}$. Hence, we get that 
\begin{align*}
    ||f_1' - f_2'||_{\ZZ_h^k}^2 &\leq \left(||f_1 - f_1'||_{\ZZ_h^k}  + ||f_2 - f_2'||_{\ZZ_h^k} + ||f_1 - f_2||_{\ZZ_h^k}\right)^2 \\ 
    &\leq \left(||f_1 - f_2||_{\ZZ_h^k}  + 2\sqrt{\delta|\ZZ_h^k|/(64T^3)}\right)^2 
    \leq \alpha/50
\end{align*}

We now consider the $Y_i$'s that are generated by $f_1', f_2'$. It holds that $||f_1' - f_2'||^2_{\ZZ_h^k} \leq \alpha/50 \implies ||f_1' - f_2'||^2_{\ZZ_h^{k-1}} \leq \alpha/50$. Since the event $\EE_h^{k-1}$ happens it follows that $\min\{||f_1' - f_2'||_{\hat{\ZZ}_h^{k-1}}^2,T(H+1)^2 \} \leq 100(\alpha/50) = 2\alpha < 2000000\alpha \implies Z_k \leq 2000000\alpha$. Thus, every $Y_i$ is exactly $(f_1'(z_h^i) - f_2'(z_h^i))^2$ multiplied by the number of times $z_h^i$ is in the sub-sampled dataset. Hence, we get
\begin{align*}
    ||f_1' - f_2'||^2_{\hat{\ZZ}_h^k} = \sum_{i=2}^k Y_i &\leq \sum_{i=2}^k\E[Y_i | \mathbb{F}_i] + \alpha/100\\
    &\leq ||f'_1 - f_2'||_{\ZZ_h^k}^2 + \alpha/100 \leq 3\alpha/100 
\end{align*}

where the first inequality follows from the concentration bound we have derived and the other two simply from the definitions of these quantities.  

We now bound $||f_1 - f_2||^2_{\hat{\ZZ}_h^k}$. We have that
\begin{align*}
    ||f_1 - f_2||_{\hat{\ZZ}_h^k}^2 &\leq \left(||f_1' - f_2'||_{\hat{\ZZ}_h^k} + ||f_1 - f_1'||_{\hat{\ZZ}_h^k} + ||f_2 - f_2'||_{\hat{\ZZ}_h^k}\right)^2  \\
    &\leq(||f_1' - f_2'||_{\hat{\ZZ}_h^k} + 2\sqrt{|\hat{\ZZ}_h^k}| \cdot \sqrt{\delta/(64T^3)})^2 \\
    &\leq (||f_1' - f_2'||_{\hat{\ZZ}_h^k} + 2)^2 \leq (\sqrt{3\alpha/100}+2)^2 \leq \alpha
\end{align*}

Hence, we have shown that $\underline{\CC}_h^k(\alpha) \subseteq \hat{\CC}_h^k(\alpha)$. So in this case, the one inequality that define $\EE_h^k$ holds.\\

We shift our attention to the second inequality now. We will show the contrapositive of our claim, i.e. if $f_1, f_2 \notin \overline{C}_h^k(\alpha) \implies f_1, f_2 \notin \hat{C}_h^k(\alpha)$. Let $f_1, f_2 \in \FF_h \times \FF_h$ such that $ ||f_1 - f_2||_{\ZZ_h^k} > 100\alpha$.  We know that there exist $f_1',f_2' \in \CC(\FF,\sqrt{\delta/(64T^3)}) \times \CC(\FF,\sqrt{\delta/(64T^3)})$ with $||f_1 - f_1'||_{\infty}, ||f_2-f_2'||_{\infty} \leq \sqrt{\delta/64T^3}$. Hence, using the triangle inequality we get that 
\begin{align*}
    ||f_1' - f_2'||_{\ZZ_h^k}^2 &\geq (||f_1-f_2||_{\ZZ_h^k} - ||f_1-f_1'||_{\ZZ_h^k} - ||f_2-f_2'||_{\ZZ_h^k})^2 \\
    &\geq 
    (||f_1-f_2||_{\ZZ_h^k} - 2 \sqrt{|\ZZ_h^k|}\sqrt{\delta/(64T^3)})^2 \\
    &= (\sqrt{100\alpha} - 2 \sqrt{\delta/(64T^2)})^2 > 50\alpha
\end{align*}

Again, consider the $Y_i$'s that are generated by $f_1', f_2'$. We want to show that $||f_1' - f_2'||_{\hat{\ZZ}_h^k}^2 > 40\alpha$. Assume towards contradiction that $||f_1' - f_2'||_{\hat{\ZZ}_h^k}^2 \leq 40\alpha$. We consider three different cases. 

\textbf{First Case: $||f_1' - f_2'||^2_{\ZZ_h^k} \leq 2000000\alpha$}. Similarly as before, we have that
\begin{align*}
    ||f_1' - f_2'||_{\hat{\ZZ}_h^k}^2 = \sum_{i=2}^k Y_i &\geq \E[Y_i| \mathbb{F}_i] - \alpha/100 \\
    &> 50\alpha - \alpha/100 > 40\alpha
\end{align*}

So we get a contradiction.

\textbf{Second Case: $||f_1' - f_2'||^2_{\ZZ_h^{k-1}} > 10000\alpha$}. The contradiction comes directly from the fact that $\EE_h^{k-1}$ holds, so $$||f_1' - f_2'||_{\hat{\ZZ}_h^k}^2 \geq ||f_1' - f_2'||_{\hat{\ZZ}_h^{k-1}}^2 > 100\alpha$$.

    \textbf{Third Case: $||f_1' - f_2'||^2_{\ZZ_h^{k-1}} \leq 10000\alpha$ and $||f_1' - f_2'||^2_{\ZZ_h^k} > 2000000\alpha$}. We can directly see that for this case $(f_1'(z_h^k) - f_2'(z_h^k))^2 \geq 1900000\alpha$. 
Since $||f_1' - f_2'||^2_{\ZZ_h^{k-1}} \leq 10000\alpha \implies ||f_1' - f_2'||^2_{\hat{\ZZ}_h^{k-1}} \leq 1000000\alpha $. 
Thus, since $\alpha \geq \beta$ we can see that the sensitivity is 1 so the element will be added to the sub-sampled dataset.
Hence, $||f_1' - f_2'||_{\ZZ_h^k}^2 \geq (f_1'(z_h^k) - f_2'(z_h^k))^2 > 40\alpha$.

Thus, in any case we have that $||f_1' - f_2'||_{\hat{\ZZ}_h^k}^2 > 40\alpha > \alpha$, so we get the result.
\end{proof}

We are now ready to prove Lemma~\ref{lem: mixture mdp: sub-sampled dataset is good approximation to original one}.

\begin{prevproof}{Lemma}{lem: mixture mdp: sub-sampled dataset is good approximation to original one}

We know that for all $k \in [K], k \neq 1, h \in [H]$ it holds that
\begin{align*}
    \Pr(\EE_h^1 \EE_h^2\ldots \EE_h^{k-1}) - \Pr(\EE_h^1 \EE_h^2\ldots \EE_h^k) &= \Pr\left( \EE_h^1 \EE_h^2\ldots \EE_h^{k-1} (\EE_h^k)^c\right) \\&=
    \Pr\left( \EE_h^1 \EE_h^2\ldots \EE_h^{k-1} \left(\cap_{n=0}^\infty\EE_h^k(100^n\beta)\right)^c\right) \\&=
    \Pr\left( \EE_h^1 \EE_h^2\ldots \EE_h^{k-1} \cup_{n=0}^\infty\EE_h^k(100^n\beta)^c\right) \\
    &\leq \sum_{n=0}^{\infty} \Pr\left( \EE_h^1 \EE_h^2\ldots \EE_h^{k-1} (\EE_h^k(100^n\beta))^c\right) \\&=
    \sum_{n \geq 0, 100^n\beta \leq T(H+1)^2} \Pr\left( \EE_h^1 \EE_h^2\ldots \EE_h^{k-1} (\EE_h^k(100^n\beta))^c\right).
\end{align*}

   Thus, using Lemma~\ref{lem:helper lemma to prove approximation of sub-sampled dataset} we see that $\Pr(\EE_h^1 \EE_h^2\ldots \EE_h^{k-1}) - \Pr(\EE_h^1 \EE_h^2\ldots \EE_h^k) \leq \delta/(32T^2)(\log(T(H+1)^2/\beta)+2) \leq \delta/32T$.
   
   Hence, for any fixed $h \in [H]$ we get
   \begin{align*}
       \Pr\left(\bigcap_{k=1}^K \EE_h^k \right) &= 1 - \sum_{k=1}^K\left(\Pr(\EE_h^1\EE_h^2\ldots\EE_h^{k-1}) - \Pr(\EE_h^1\EE_h^2\ldots\EE_h^{k})\right)\\
       &\geq 1 - K (\delta/32T) = 1 - \delta/(32H)
   \end{align*}

and by taking a union bound over $h\in [H]$ we get the result.
\end{prevproof}

Now that we have shown that the sub-sampled dataset approximates well the original one, we shift our attention back to showing that our approach
achieves optimism.

We first need a definition and a concetration lemma 
that is related to least-squares-estimators from prior
work.
\begin{definition}
\label{def:subgaussian-variable}
A random variable $X$ is conditionally $\sigma$-subgaussian
with respect to some filtration $\mathbb{F}$ if for all $\lambda \in \R$
it holds that $\mathbb{E}[\exp(\lambda X)] \leq \exp(\lambda^2 \sigma^2/2)$.
\end{definition}




\begin{lemma}[\citep{russo2013eluder}, \citep{ayoub2020model}]\label{lem:helper lemma from prior work, concentration of LSE}
Let $\mathbb{F} = \{\mathbb{F}_p \}_{p=0,1,\ldots}$ be a filtration, $\{(X_p, Y_p)\}_p$ measurable random variables where $X_p \in \mathcal{X}, Y_p \in \R$. Let $\tilde{\FF}$ be a set of measurable functions from $\mathcal{X}$ to $\R$ and assume that $\E[Y_p| \mathbb{F}_{p-1}] = f^*(X_p)$ for some $f^* \in \tilde{\FF}$. Assume that $\{Y_p - f^*(X_p)\}_{p=1,\ldots}$ is conditionally $\sigma$-subgaussian given $\mathbb{F}_{p-1}$. Let $\hat{f}_t = \arg\min_{f\in \tilde{\FF}}\sum_{p=1}^t\left(f(X_p) - Y_p \right)^2$ and $\tilde{\FF}_t(\beta) = \left\{ f \in \tilde{F}: \sum_{p=1}^t \left(f(X_p) - \hat{f}(X_p)) \right)^2 \leq \beta \right\}$. Then, for any $\alpha > 0$, with probability $1-\delta$, for all $t \geq 1$ it holds that $f^* \in \tilde{\FF}_t(\beta_t(\delta, \alpha))$, where
\begin{align*}
    \beta_t(\delta, \alpha) = 8\sigma^2\log(2\NN(\tilde{\FF},\alpha)/\delta) + 4t\alpha\left(C+ \sqrt{\sigma^2\log(4t(t+1)/\delta)} \right).
\end{align*}
\end{lemma}



We are now ready to prove that our algorithm ensures optimism.

\begin{lemma}\label{lem:mixture mdp bound between the difference of estimate of Q and Q induced by policy}
With probability at least $1-2\delta$, we have that for all $h \in [H], k \in [K], s \in \S, a \in \AA$
\begin{align*}
    Q_h^k(s,a) - Q_h^{\pi_k}(s,a) \leq \langle P_h(\cdot|s,a), V_{h+1}^k(\cdot) - V_{h+1}^{\pi_k}(\cdot)\rangle + 2b_h^k(s,a).
\end{align*}

Moreover, it holds that $Q_h^k(s,a) \geq Q_h^*(s,a)$.
\end{lemma}

\begin{proof}
Fix some $h \in [H], k \in [K], s \in \S, a \in \AA$. Throughout the proof, we condition on the events in Lemma~\ref{lem: mixture mdp: sub-sampled dataset is good approximation to original one} and Lemma~\ref{lem:helper lemma from prior work, concentration of LSE}.
We assume that $k$ is a round that we perform an update.

We define $\XX = \S \times \AA \times \VV, X_h^k = (s_h^k, a_h^k, V_{h+1}^k(\cdot)), Y_h^k = V_{h+1}^k(s^k_{h+1})$. We also pick $\tilde{\FF}= \FF_h$, where $\FF_h$ is defined
in Definition~\ref{def:model-based function class}. Then, we see that $\mathbb{E}[Y_h^k|\mathbb{F}_{k-1}] = f^*_h(X_h^k)$, where $f^*_h$ is the function
that corresponds to the true model $P_h$, and we know that $f^*_h \in \FF_h$. 
Recall that the optimization problem we solve in Algorithm~\ref{alg:main algorithm} for every round $k$ we 
update our policy is 
$$\hat{P}_h^k = \arg\min_{P \in \PP_h} \sum_{p=1}^k \left( \langle P(\cdot| s^p_h, a^p_h), V_{h+1}^p \rangle - V_{h+1}^p(s^p_{h+1})\right)^2.$$
Based on the definition of $\FF_h$, we can see that this is equivalent to
$$f_h^k = \arg\min_{f \in \FF_h} \sum_{p=1}^k \left( f(s_h^p,a_h^p,V_{h+1}^p) - V_{h+1}^p(s^p_{h+1})\right)^2.$$
Moreoever, $Y_h^k \in [0,H]$, so $Z_h^k = Y_h^k - f_h^*(X_h^k)$ is $H/2$-conditionally subgaussian. 
Thus, if we pick $\alpha = 1/T$ and $$\beta_h^k = 4H^2\log(2\NN(\FF_h,1/T)/\delta) + 4k/T\left(C + \sqrt{H^2/4\log(4(k(k+1)/\delta)})\right)$$ then Lemma~\ref{lem:helper lemma from prior work, concentration of LSE} gives us that $||f^*_h - \hat{f}_h^k||^2_{\ZZ_h^k} \leq \beta_h^k$. 
In particular,
can pick $$\tilde{\beta} = \beta_h^K = 4H^2\log(2\NN(\FF_h,1/T)/\delta) + 4K/T\left(C + \sqrt{H^2/4\log(4(K(K+1)/\delta)})\right)$$
and get a parameter that is independent of $k$. Moreover, Lemma~\ref{lem: mixture mdp: sub-sampled dataset is good approximation to original one}
$$||f_h^* - \hat{f}_h^k||_{\hat{\ZZ}_h^k} \leq 100\tilde{\beta} = \beta.$$
This implies that for our bonus function we have that $|\hat{f}_h^k(s,a,V_{h+1}^k(\cdot)) - f_h^*(s,a,V_{h+1}^k(\cdot))| \leq b_h^k(s,a)$.

Hence, we have that
\begin{align*}
    \langle \hat{P}_h^k(\cdot| s,a), V_{h+1}^k(\cdot) \rangle - \langle P_h(\cdot|s,a), V_{h+1}^k(\cdot) \rangle &= \hat{f}_h^k(s,a,V_{h+1}^k) - f^*_h(s,a,V_{h+1}^k) \\&
    \leq 
    |\hat{f}_h^k(s,a,V_{h+1}^k) - f^*_h(s,a,V_{h+1}^k) | \\&
    \leq b_h^k(s,a).
\end{align*}

Now we use the definition of $Q_h^k(s,a), Q_h^{\pi_k}(s,a)$ to get that
\begin{align*}
    Q^k_h(s, a) &\leq r_h^k(s,a) + \langle \hat{P}_h^k(\cdot|s,a), V_{h+1}^k(\cdot) \rangle + b_h^k(s,a)  \\
    &\leq r_h^k(s,a) + \langle P_h(\cdot|s,a), V_{h+1}^k(\cdot) \rangle + 2b_h^k(s,a)\\
    Q^{\pi_k}_h(s, a) &= r_h^k(s,a) + \langle P_h(\cdot|s,a), V_{h+1}^{\pi_k}(\cdot) \rangle.
\end{align*}

Combining these two, we get that
\begin{align*}
    Q^k_h(s, a) - Q^{\pi_k}_h(s, a) \leq \langle P_h(\cdot|s,a), V_{h+1}^k(\cdot) - V_{h+1}^{\pi_k}(\cdot)\rangle + 2b_h^k(s,a)
\end{align*}

which proves the first part of the result.

For the second part, notice that if $Q_h^k(s,a) = H$ then the statement holds trivially since $Q^*_h(s,a) \leq H$. So we can assume without loss of generality that $Q_h^k(s,a) = r_h^k(s,a) + \langle \hat{P}_h^k(\cdot|s,a), V_{h+1}^k(\cdot) \rangle + b_h^k(s,a)$. The Bellman optimality condition gives us that $Q^*_h(s,a) = r_h^k(s,a) + \langle P_h^*(\cdot|s,a), V_{h+1}^*(\cdot) \rangle $. Hence, we have that
\begin{align*}
    Q_h^k(s,a) - Q_h^*(s,a) &= \langle \hat{P}_h^k(\cdot|s,a), V_{h+1}^k(\cdot) \rangle - \langle P_h(\cdot|s,a), V_{h+1}^*(\cdot) \rangle + b_h^k(s,a) \\
    &=
    \langle \hat{P}_h^k(\cdot|s,a), V_{h+1}^k(\cdot) \rangle -\langle P_h(\cdot|s,a), V_{h+1}^k(\cdot) \rangle + \langle P_h(\cdot|s,a), V_{h+1}^k(\cdot) \rangle- \langle P_h(\cdot|s,a), V_{h+1}^*(\cdot) \rangle + b_h^k(s,a)  \\
    &= \langle \hat{P}_h^k(\cdot|s,a) - P_h(\cdot|s,a), V_{h+1}^k(\cdot)  \rangle + \langle P_h(\cdot|s,a), V_{h+1}^k(\cdot) - V_{h+1}^*(\cdot)  \rangle + b_h^k(s,a).
\end{align*}

Now from our previous discussion it follows that $b_h^k(s,a)  + \langle \hat{P}_h^k(\cdot|s,a) - P_h(\cdot|s,a), V_{h+1}^k(\cdot)  \rangle \geq 0$. Hence, it suffices to show that $\langle P_h(\cdot|s,a), V_{h+1}^k(\cdot) - V_{h+1}^*(\cdot)  \rangle \geq 0$. To do that, we can just prove that $V_{h+1}^k(s') - V_{h+1}^*(s') \geq 0, \forall s' \in \S$. Since $V_{H+1}^k(s') = V_{H+1}^*(s') = 0, \forall s' \in \S$ we get that $Q_H^k(s,a) \geq Q_H^*(s,a)$. Thus, if we combine this with the update rule for $V_h^k, V_h^*$ we get the claim by induction.
\end{proof}

Now that we have established the previous lemma, we need
to bound the bonus that we are using in every round.
The issue is that we do not update our policy in every round.

To do that, we follow a similar approach as in Appendix~\ref{sec:linear-mdp-proof}.

\begin{lemma}
\label{lem:model-based-bound-bonus}
For every set $K' \subseteq [K]$ With probability at least $1-\delta$, we have that
\begin{align*}
    \sum_{i=1}^{|K'|}\sum_{h=1}^H b_h^{k_i}(s_h^{k_i}, a_h^{k_i}) \leq H + H(H+1)\dim_E(\FF, 1/T) + C H\sqrt{\textnormal{dim}_E(\FF, 1/T) |K'|\beta}
\end{align*}

where $\textnormal{dim}_E(\FF, 1/T) = \max_{h \in [H]} \textnormal{dim}_E(\FF_h, 1/T)$.

\end{lemma}

\begin{proof}
We condition on the event described in Lemma~\ref{lem: mixture mdp: sub-sampled dataset is good approximation to original one}. From the definition of the bonus function, we have that for any $k \in [K]$
\begin{align*}
    b_h^k(s_h^k, a_h^k) \leq \Bar{b}_h^k(s_h^k, a_h^k) = \sup_{||f_1 - f_2||^2_{\ZZ_h^k} \leq 100\beta}|f_1(s_h^k, a_h^k) - f_2(s_h^k, a_h^k)|.
\end{align*}

We bound $\sum_{i=1}^{K'}\Bar{b}_h^{k_i}(s_h^{k_i}, a_h^{k_i})$ for each $h \in [H]$ separately.\\

Given some $\epsilon > 0$, we define $\tilde{K} = \{k \in K': \Bar{b}_h^{k}(s_h^{k}, a_h^{k}) > \epsilon\}$, 
i.e. the set of episodes in $K'$ where the bonus function at $h$ 
 has value greater than $\epsilon$. Consider some $k \in K'$. We denote 
 $\LL_h = \{(s_h^k, a_h^k, V_{h+1}^k(\cdot)): k \in \tilde{K} \}, L_h = |\LL_h|$, and $N = L_h/\textnormal{dim}_E(\FF_h,\epsilon) - 1$. 
Our goal is to show that there is some $z_h^k = (s_h^k, a_h^k, V_{h+1}^k(\cdot)) \in \LL_h$ 
that is $\epsilon$-dependent on at least $N$ disjoint subsequences in $\ZZ_h^k \cap \LL_h$. 

To do that, we decompose $\LL_h$ into $N+1$ disjoint subsets and we denote the $j$-th subset by $\LL_{h,j}$. 
We use the following procedure. Initially we set $\LL_{h,j} = \emptyset$ for all $j \in [N+1]$
and consider every $z_h^k \in \LL_{h}$ in a sequential manner. 
For each such $z_h^k$ we find the smallest index $j, 1 \leq j \leq N$, 
such that $z_h^k$ is $\epsilon$-independent of the elements in $\LL_{h,j}$ with
respect to $\FF_h$. If there is no such $j$, we set $j=N+1$. 
Then, we update $\LL_{h,j} \leftarrow \LL_{h,j} \cup z_h^k$.
Notice that after we go through all the elements of $\LL_{h}$,
we must have that $\LL_{h,N+1} \neq \emptyset$. 
This is because every set $\LL_{h,j}, 1 \leq j \leq N$, contains at most 
$\textnormal{dim}_E(\FF_h, \epsilon)$ elements. 
Moreover, by definition, every element $z_h^k \in \LL_{h,N+1}$ is
$\epsilon$-dependent on at least $N$ disjoint
subsequences in $\LL_{h}$.

Furthermore, since $\Bar{b}_h^{k}(s_h^k, a_h^k) > \epsilon$ for all $z_h^k \in \LL_{h}$
there must exist $f_1, f_2 \in \FF_h$ such that $|f_1(s_h^k, a_h^k, V_{h+1}^k(\cdot)) - f_2(s_h^k, a_h^k, V_{h+1}^k(\cdot))| > \epsilon$ and
$||f_1 - f_2||^2_{\ZZ_h^k} \leq 100\beta$. Hence, since $z_h^k \in \LL_{h,N+1}$ is
$\epsilon$-dependent on $N$ disjoint subsequences $\LL_{h}$ and for each such
subsequence $\LL$, by the definition of $\epsilon$-dependence, it holds that $||f_1 - f_2||^2_{\LL} > \epsilon^2$ we have that
\begin{align*}
    N \epsilon^2 \leq ||f_1 - f_2||^2_{\ZZ_h^k} \leq 100\beta\\
    \implies
    (L_h/\textnormal{dim}_E(\FF_h,\epsilon)-1)\epsilon^2 \leq 100\beta \\
    \implies L_h \leq \left( \frac{100\beta}{\epsilon^2} + 1\right)\textnormal{dim}_E(\FF_h, \epsilon).
\end{align*}

We now pick a permutation $\Bar{b}_1 \geq \Bar{b}_2 \geq \ldots \geq \Bar{b}_{|K'|}$ of the bonus
functions $\{\Bar{b}_h^k(s^k_h, a^k_h) \}_{k \in K'}$. For all $\Bar{b}_k \geq 1/|K'|$ it holds
that
\begin{align*}
    k \leq \left(\frac{100\beta}{\Bar{b}_k^2} + 1\right)\textnormal{dim}_E(\FF_h, \Bar{b}_k) &\leq \left(\frac{100\beta}{\Bar{b}_k^2} + 1\right)\textnormal{dim}_E(\FF_h, 1/K') \implies \\
    \Bar{b}_k &\leq \left(\frac{k}{\textnormal{dim}_E(\FF_h, 1/K')} -1  \right)^{-1/2}\sqrt{100\beta}.
\end{align*}

Moreover, notice that we get by definition that $\Bar{b}_k \leq H + 1$. Hence, we have that
\begin{align*}
    \sum_{i=1}^{|K'|} \Bar{b}_h^{k_i}(s_h^{k_i}, a_h^{k_i}) &= \sum_{i: \Bar{b}_{k_i} < 1/|K'|} \Bar{b}_h^{k_i}(s_h^{k_i}, a_h^{k_i}) + \sum_{i: \Bar{b}_{k_i} \geq 1/|K'|} \Bar{b}_h^{k_i}(s_h^{k_i}, a_h^{k_i}) \\&\leq
    |K'| \cdot 1/|K'| + \sum_{i: \Bar{b}_{k_i} \geq 1/|K'|, i \leq  \textnormal{dim}_E(\FF_h, 1/|K'|)} \Bar{b}_h^{k_i}(s_h^{k_i}, a_h^{k_i}) + \sum_{i: \Bar{b}_{k_i} \geq 1/|K'|, |K'| \geq i >  \textnormal{dim}_E(\FF_h, 1/|K'|)} \Bar{b}_h^{k_i}(s_h^{k_i}, a_h^{k_i}) \\&\leq 
    1 + (H+1)\dim_E(\FF_h, 1/|K'|) + \sum_{|K'| \geq i >  \textnormal{dim}_E(\FF_h, 1/|K'|)} \left(\frac{k}{\textnormal{dim}_E(\FF_h, 1/|K'|)} -1  \right)^{-1/2}\sqrt{100\beta} \\&\leq
    1 + (H+1)\dim_E(\FF_h, 1/|K'|) + C \sqrt{\textnormal{dim}_E(\FF_h, 1/|K'|) |K'|\beta} \\&\leq
    1 + (H+1)\dim_E(\FF_h, 1/T) + C \sqrt{\textnormal{dim}_E(\FF_h, 1/T) |K'|\beta}
\end{align*}

for some constant $C > 0$, where the second to last inequality can be obtained by bounding the summation by the integral and the last one by the definition of the eluder dimension. We get the final result by summing up all the inequalities over $H$.
\end{proof}

The next step in our proof, is to bound the number of episodes that our policy
can be worse than the optimal one
by $2^n\gap_{\min}$, for all $n \in \N$. This
is inspired by \cite{he2021logarithmic}.

\begin{lemma}\label{lem:indicator bound in each interval for linear mixture model}
If we pick 
$$\beta = 4H^2\log(2\NN(\FF,1/T)/\delta) + 4/H\left(C + \sqrt{H^2/4\log(4(K(K+1)/\delta)})\right),$$
then for every $h \in [H]$ and $n \in \N$, with probability at least $1 - (K+3)\delta$, we have that
\begin{align*}
    \sum_{k=1}^K \mathbbm{1}\left[V_h^*(s_h^k) - Q_h^{\pi_k}(s_h^k,a_h^k) \geq 2^n \gap_{\min} \right] \leq \frac{\tilde{C} H^4 \log(T\NN(\FF_h,1/T)/\delta) \dim^2_E(\FF, 1/T) }{4^n \gap_{\min}^2}.
\end{align*}
 
 \end{lemma}

 \begin{proof}
 We keep $h$ fixed.

We denote by $K'$ the set of episodes where the gap at step $h$ is at least $2^n$, i.e.
\begin{align*}
    K' = \left\{k \in [K]: V^*_h(s_h^k) - Q_h^{\pi_k}(s_h^{k}, a_h^{k}) \geq 2^n \gap_{\min}\right\}.
\end{align*}

The goal is to bound the quantity $\sum_{i=1}^{|K'|}\left(Q_h^{k_i}(s_h^{k_i},a_h^{k_i}) - Q_h^{\pi_{k_i}}(s_h^{k_i}, a_h^{k_i})\right)$ from below and above with functions
$f_1(|K'|), f_2(|K'|)$ and then use the fact that $f_1(|K'|) \leq f_2(|K'|)$ to derive an upper bound on $|K'|$.

For the lower bound, we have that
\begin{align*}
    \sum_{i=1}^{|K'|}\left(Q_h^{k_i}(s_h^{k_i},a_h^{k_i}) - Q_h^{\pi_{k_i}}(s_h^{k_i}, a_h^{k_i})\right) &\geq \sum_{i=1}^{|K'|}\left(Q_h^{k_i}(s_h^{k_i},\pi_h^*(s_h^{k_i})) - Q_h^{\pi_{k_i}}(s_h^{k_i}, a_h^{k_i})\right) \\
    &\geq \sum_{i=1}^{|K'|}\left(Q_h^*(s_h^{k_i},\pi_h^*(s_h^{k_i})) - Q_h^{\pi_{k_i}}(s_h^{k_i}, a_h^{k_i})\right) 
    \\&= \sum_{i=1}^{|K'|}\left(V_h^*(s_h^{k_i}) - Q_h^{\pi_{k_i}}(s_h^{k_i}, a_h^{k_i})\right) \\
    &\geq 2^n \gap_{\min} |K'|
\end{align*}

where the first inequality holds by the definition of the policy $\pi_{k_i}$, the second one follows because $Q^{k_i}_h(\cdot, \cdot)$ is an optimistic estimate of $Q^*_h(\cdot, \cdot)$ which happens with probability at least $1-2\delta$ (see Lemma~\ref{lem:mixture mdp bound between the difference of estimate of Q and Q induced by policy}) and the third one by the definition of $k_i$.

We get the upper bound on this quantity in the following way. For any $h' \in [H]$ we have
\begin{align*}
    Q_{h'}^k(s_{h'}^k, a_{h'}^k) - Q_{h'}^{\pi_k}(s_{h'}^k, a_{h'}^k) &\leq \sum_{s'\in \S}P_{h'}(s'|s_{h'}^k,a_{h'}^k)V^k_{h'+1}(s') + r_{h'}(s_{h'}^k, a_{h'}^k) + 2b_{h'}^k(s_{h'}^k, a_{h'}^k) - Q_{h'}^{\pi_k}(s_{h'}^k, a_{h'}^k) \\&=
    \left\langle P_{h'}(\cdot| s_{h'}^k, a_{h'}^k), V_{h'+1}^k - V_{h'+1}^{\pi_k} \right\rangle + 2b_{h'}^k(s_{h'}^k, a_{h'}^k) \\
    &= V_{h'+1}^k(s_{h'+1}^k) - V_{h'+1}^{\pi_k}(s_{h'+1}^k) + \epsilon_{h'}^k + 2b_{h'}^k(s_{h'}^k, a_{h'}^k) \\&= 
    Q_{h'+1}^k(s_{h'+1}^k, a_{h'+1}^k) - Q_{h'+1}^{\pi_k}(s_{h'+1}^k, a_{h'+1}^k) + \epsilon_{h'}^k + 2b_{h'}^k(s_{h'}^k, a_{h'}^k)
\end{align*}
where we define $\epsilon_{h'}^k = \left\langle P_{h'}(\cdot| s_{h'}^k, a_{h'}^k), V_{h'+1}^k - V_{h'+1}^{\pi_k} \right\rangle - \left(V_{h'+1}^k(s_{h'+1}^k) - V_{h'+1}^{\pi_k}(s_{h'+1}^k)\right)$ and the inequality follows from Lemma~\ref{lem:mixture mdp bound between the difference of estimate of Q and Q induced by policy}.

We now take the summation over all $k \in |K'|, h \leq h' \leq H$ and we get
\begin{align*}
    \sum_{i=1}^{|K'|}\sum_{h'=h}^H \left( Q_h^{k_i}(s_h^{k_i}, a_h^{k_i}) - Q_h^{\pi_{k_i}}(s_h^{k_i}, a_h^{k_i}) \right) \leq \sum_{i=1}^{|K'|} \sum_{h'=h}^H\epsilon_{h'}^{k_i} + \sum_{i=1}^{|K'|} \sum_{h'=h}^H b_{h'}^{k_i}(s_{h'}^{k_i}, a_{h'}^{k_i}).
\end{align*}

We will bound each of the two terms on the RHS separately. 

For the first term, we notice that $x_j = \left\langle P_{j}(\cdot| s_{j}^{k_i}, a_{j}^{k_i}), V_{j+1}^{k_i} - V_{j+1}^{\pi_{k_i}} \right\rangle - \left(V_{j+1}^{k_i}(s_{j+1}^{k_i}) - V_{j+1}^{\pi_{k_i}}(s_{j+1}^{k_i})\right)$ forms a martingale difference sequence with zero mean and $|x_j| \leq 2H$. Hence, we can use Lemma~\ref{lem:azuma-hoeffding inequality} and  that for each $k\in K'$, with probability at least $1-\delta$ we have that
\begin{align*}
    \sum_{i=1}^k \sum_{j=1}^H \left( \left\langle P_{j}(\cdot| s_{j}^{k_i}, a_{j}^{k_i}), V_{j+1}^{k_i} - V_{j+1}^{\pi_{k_i}} \right\rangle - \left(V_{j+1}^{k_i}(s_{j+1}^{k_i}) - V_{j+1}^{\pi_{k_i}}(s_{j+1}^{k_i})\right) \right) \leq \sqrt{8k H^3 \log(1/\delta)}.
\end{align*}

If we take the union bound over all $k\in [K]$ we have that with probability at least $1-|K'|\delta$
\begin{align*}
    \sum_{i=1}^{|K'|} \sum_{h'=h}^H\epsilon_{h'}^{k_i} \leq \sqrt{8 |K'| H^3 \log(1/\delta)}.
\end{align*}

We now focus on the second term. Using Lemma~\ref{lem:model-based-bound-bonus} we get that
\begin{align*}
    \sum_{i=1}^{|K'|} \sum_{h'=h}^H b_{h'}^{k_i}(s_{h'}^{k_i}, a_{h'}^{k_i}) \leq H + H(H+1)\dim_E(\FF, 1/T) + C H\sqrt{\textnormal{dim}_E(\FF, 1/T) |K'|\beta}
\end{align*}

and this happens with probability at least $1-\delta$. Hence, combining the upper and lower bound of $$\sum_{i=1}^{|K'|}\left(Q_h^{k_i}(s_h^{k_i},a_h^{k_i}) - Q_h^{\pi_{k_i}}(s_h^{k_i}, a_h^{k_i})\right)$$ we get that
\begin{align*}
    2^n \gap_{\min}|K'| &\leq \sqrt{8 |K'| H^3 \log(1/\delta)} +  H + H(H+1)\dim_E(\FF, 1/T) + C H\sqrt{\textnormal{dim}_E(\FF, 1/T) |K'|\beta}   \\
     \\
     \implies |K'| &\leq \frac{\tilde{C} H^4 \log(T\NN(\FF,1/T)/\delta) \dim^2_E(\FF, 1/T) }{4^n \gap_{\min}^2}.
\end{align*}
 \end{proof}
 
 We are now ready to prove Lemma~\ref{lem:linear mixture sum of suboptimalities over all K for some h}.
 
 \begin{prevproof}{Lemma}{lem:linear mixture sum of suboptimalities over all K for some h}
Throughout this proof we condition on the event described in Lemma~\ref{lem:indicator bound in each interval for linear mixture model} which happens with probability at least $1-(K+3)\delta$. Since $\gap_{\min} > 0$ whenever we do not take the optimal action, we have that either $V_h^*(s_k) - Q_h^*(s_h^k, a_h^k) = 0$ or $V_h^*(s_k) - Q_h^*(s_h^k, a_h^k) \geq \gap_{\min}$. Our approach is to divide the interval $[0, H]$ into $N = \lceil \log(H/\gap_{\min})\rceil$ intervals and count the number of $V_h^*(s_k) - Q_h^*(s_h^k, a_h^k)$ that fall into each interval. Notice that for every $V_h^*(s_k) - Q_h^*(s_h^k, a_h^k)$ that falls into interval $i$ we can get an upper bound of $V_h^*(s_k) - Q_h^*(s_h^k, a_h^k) \leq 2^i\gap_{\min}$ and this upper bound is essentially tight. Hence, we have that
\begin{align*}
    \sum_{k=1}^K \left( V_h^*(s_h^k) - Q_h^*(s_h^k, a_h^k) \right) &\leq \sum_{i=1}^N \sum_{k=1}^K \mathbbm{1}\left[ 2^i\gap_{\min} \geq V_h^*(s_h^k) - Q_h^*(s_h^k, a_h^k) \geq 2^{i-1}\gap_{\min} \right] \cdot 2^i \gap_{\min} \\
    &\leq \sum_{i=1}^N \sum_{k=1}^K \mathbbm{1}\left[  V_h^*(s_h^k) - Q_h^*(s_h^k, a_h^k) \geq 2^{i-1}\gap_{\min} \right] \cdot 2^i \gap_{\min} \\
    &\leq \sum_{i=1}^N \frac{\tilde{C} H^6 \log(T\NN(\FF,1/T)/\delta) \dim^2_E(\FF, 1/T) }{4^{i-1} \gap_{\min}^2} \cdot 2^i \gap_{\min} \\
    &= \sum_{i=1}^N \frac{C' H^4 \log(T\NN(\FF,1/T)/\delta) \dim^2_E(\FF, 1/T) }{2^i \gap_{\min}} \\
    &\leq
    \frac{C H^4 \log(T\NN(\FF,1/T)/\delta) \dim^2_E(\FF, 1/T) }{ \gap_{\min}} 
\end{align*}

where the first inequality holds by the definition of the intervals, the second due to the properties of the indicator function, the third because of Lemma~\ref{lem:indicator bound in each interval for linear mixture model} and in the last two steps we just manipulate the constants. 
\end{prevproof}
\end{document}